\documentclass[11pt]{article}
\usepackage[
letterpaper,
top=1in,
bottom=1in,
left=1in,
right=1in]{geometry}

\usepackage{microtype}
\usepackage{graphicx}
\usepackage{subcaption}
\usepackage{booktabs}

\usepackage{amsmath,amssymb,amsfonts,amsthm,turnstile}
\usepackage{aliascnt}
\usepackage{mathtools}
\usepackage{dsfont}

\usepackage{enumitem}
\usepackage{xcolor}

\usepackage{newtxtext,newtxmath}
\definecolor{deeprednew}{rgb}{0.58, 0.0, 0.0}
\definecolor{deepgreennew}{rgb}{0.0, 0.31, 0.11}
\usepackage{hyperref}
\hypersetup{
colorlinks=true,
urlcolor=blue,
linkcolor=deeprednew,
citecolor=deepgreennew,
}

\newcommand{\veps}{\varepsilon}

\newcommand{\KL}{D_{\mathrm{KL}}}
\newcommand{\TV}{d_{\mathrm{TV}}}

\usepackage{algorithm}
\usepackage{algorithmic}

\usepackage{tabularx}
\usepackage{makecell}
\usepackage{xcolor}
\usepackage{amssymb}

\title{Learning Gaussian Graphical Models \\ from a Glauber Trajectory Without Mixing}
\author{Eric Shen\thanks{Equal contribution.}\\
\texttt{eys@mit.edu}\\
MIT
\and
Tony Wu\footnotemark[1]\\
\texttt{tonyyzwu@mit.edu}\\
MIT
\and 
Mahbod Majid\\
\texttt{mahbod@mit.edu}\\
MIT
\and
Ankur Moitra\thanks{
Supported in part by DARPA expMath, a Microsoft Trustworthy AI Grant, NSF-CCF 2430381, an ONR grant, and a David
and Lucile Packard Fellowship.}\\
\texttt{moitra@mit.edu}\\
MIT}

\usepackage[capitalize,noabbrev]{cleveref}

\theoremstyle{plain}
\newtheorem{theorem}{Theorem}[section]

\newaliascnt{proposition}{theorem}
\newtheorem{proposition}[proposition]{Proposition}
\aliascntresetthe{proposition}

\newaliascnt{lemma}{theorem}
\newtheorem{lemma}[lemma]{Lemma}
\aliascntresetthe{lemma}

\newaliascnt{fact}{theorem}
\newtheorem{fact}[fact]{Fact}
\aliascntresetthe{fact}

\newaliascnt{corollary}{theorem}
\newtheorem{corollary}[corollary]{Corollary}
\aliascntresetthe{corollary}

\theoremstyle{definition}
\newaliascnt{definition}{theorem}
\newtheorem{definition}[definition]{Definition}
\aliascntresetthe{definition}

\newaliascnt{assumption}{theorem}

\aliascntresetthe{assumption}

\theoremstyle{remark}
\newaliascnt{remark}{theorem}
\newtheorem{remark}[remark]{Remark}
\aliascntresetthe{remark}

\crefname{theorem}{theorem}{theorems}
\Crefname{theorem}{Theorem}{Theorems}
\crefname{proposition}{proposition}{propositions}
\Crefname{proposition}{Proposition}{Propositions}
\crefname{lemma}{lemma}{lemmas}
\Crefname{lemma}{Lemma}{Lemmas}
\crefname{fact}{fact}{facts}
\Crefname{fact}{Fact}{Facts}
\crefname{corollary}{corollary}{corollaries}
\Crefname{corollary}{Corollary}{Corollaries}
\crefname{definition}{definition}{definitions}
\Crefname{definition}{Definition}{Definitions}
\crefname{assumption}{assumption}{assumptions}
\Crefname{assumption}{Assumption}{Assumptions}
\crefname{remark}{remark}{remarks}
\Crefname{remark}{Remark}{Remarks}
\crefname{problem}{problem}{problems}
\Crefname{problem}{Problem}{Problems}

\usepackage{thm-restate}

\begin{document}

\maketitle

\begin{abstract}
We study the task of learning the structure of a $d$-sparse Gaussian graphical model on $n$ variables from a single trajectory of Glauber dynamics. Beyond algorithmic considerations, many applications present temporally correlated observations rather than i.i.d.\ samples. In the classical i.i.d.\ setting, under comparably general sparsity and minimum edge-strength assumptions, sublinear-in-$n$ sample guarantees are known, but achieving them in polynomial-time remains open. Motivated in part by this gap, we give a polynomial-time algorithm that recovers the conditional-independence graph from a single Glauber trajectory, with a trajectory-length guarantee that does not depend on the mixing time.

Technically, our algorithm has three components. First, we estimate the conditional variances and rescale the trajectory to reduce to the unit-diagonal case, without changing the underlying graph. Second, we design a local edge test that extracts adjacency information from short update windows by isolating pairwise influence. Third, we aggregate these local statistics using a robust median-based estimator, and prove accuracy despite temporal dependence arising from a single trajectory.
\end{abstract}

\newpage
\tableofcontents
\newpage
\section{Introduction}

A \emph{Gaussian Graphical Model} (GGM) on $n$ vertices is a mean-zero Gaussian random variable
$X \sim \mathcal{N}(0,\Sigma)$.
The relevant object for the graph structure is the \emph{precision matrix} $\Theta \coloneq \Sigma^{-1}$.
We associate to $\Theta$ an undirected graph $G=(V,E)$ by putting an edge $(i,j)$ whenever $\Theta_{ij}\neq 0$.
The key fact is that, for Gaussians, zeros in $\Theta$ exactly encode conditional independences:
for distinct $i,j \in V$,
\[
X_i \perp X_j \, \vert \, X_{V\setminus\{i,j\}}
\quad\Longleftrightarrow\quad
\Theta_{ij}=0.
\]
This is known as the Markov property. An important measure of complexity of GGMs is \emph{sparsity}: we say the GGM is $d$-sparse if every vertex has at most $d$
neighbors in $G$, equivalently each row of $\Theta$ has at most $d$ nonzero off-diagonal entries.

GGMs provide a natural way to represent \emph{conditional dependence structure}
among many interacting variables. The literature on GGM applications is too vast
to survey here, but representative examples include neuroscience and brain
connectivity \cite{dyrba2020gaussian,huang2010learning}, genomics
\cite{yi2022information}, metabolic pathway reconstruction
\cite{krumsiek2011gaussian}, climate science \cite{zerenner2014gaussian},
financial systemic-risk modeling \cite{cerchiello2016conditional}, and environmental
psychology \cite{bhushan2019using}.
A recurring regime in such applications is high dimensionality, where the number
of variables $n$ can be comparable to or larger than the number of available
observations, motivating our focus on sparse GGMs, in which the conditional-independence graph has maximum degree at most $d$.

Algorithmically, the main challenge is already present in \emph{structure learning}:
recovering the edge set $E$ (equivalently, the support of $\Theta$).
Indeed, once $G$ is known, estimating the coefficients reduces to running $n$ low-dimensional
(regression) problems, each involving only the $d$ neighbors of a node. More precisely, for each node $i\in V$, the conditional distribution of $X_i$ given the remaining coordinates is
\[
X_i = -\sum_{j\in N(i)} \frac{\Theta_{ij}}{\Theta_{ii}} X_j + \xi_i,
\quad
\xi_i \sim \mathcal{N}\left(0,\frac{1}{\Theta_{ii}}\right).
\]
Thus, given $G$, estimating $\Theta$ reduces to $n$ regressions of $X_i$ onto $\{X_j : j\in N(i)\}$,
which recover the coefficients $\{-\Theta_{ij}/\Theta_{ii}\}_{j\in N(i)}$ and the noise variance $1/\Theta_{ii}$.

In the classical i.i.d.\ data model, Misra, Vuffray and Lokhov \cite{misra2020information} studied the information-theoretic
sample complexity of learning sparse GGMs \emph{without} assuming bounded spectrum or incoherence.
The only assumption they make is the following guarantee on the minimum normalized edge strength
\begin{equation*}
    \frac{\lvert\Theta_{ij}\rvert}{\sqrt{\Theta_{ii}\Theta_{jj}}} \ge \alpha
    \qquad\forall (i, j) \in E,
    \tag{non-degeneracy}
\label{eq:non-degeneracy}
\end{equation*}
which is a natural non-degeneracy condition ensuring that present edges are not arbitrarily weak. It is important to note that this constraint does \emph{not} impose any assumptions on the minimum eigenvalue of the normalized matrix, and the spectrum may be arbitrarily ill-conditioned. For a simple demonstration of this, see appendix \cref{app:nondegeneracy-conditioning}.

They show that information-theoretically $O(d \log n / \alpha^2)$ i.i.d\ samples suffice for learning the graph structure. Earlier work of Wang, Wainwright, and Ramchandran \cite{wang2010information} shows that at least $\Omega(\log n / \alpha^2)$ i.i.d\ samples are necessary for this task, and it is currently unknown which of the upper bound or the lower bound is tight.
However, the price paid is computational: the algorithm of \cite{misra2020information} uses an exhaustive search based on an $\ell_0$-constrained sparse linear regression and runs in time \(n^{\Omega(d)}\). Whether one can match the information-theoretic sample complexity with a polynomial-time algorithm for general GGMs remains open. More broadly, there is evidence for computational barriers in related sparse linear regression problems. In particular, in the fixed-design, worst-case setting, \emph{proper} sparse linear regression---meaning the algorithm must output a \(k\)-sparse predictor when a \(k\)-sparse solution exists---is \textsf{NP}-hard \cite{natarajan1995sparse,zhang2014lower}.

In many scientific settings, observations are not i.i.d.; instead we observe a system evolving over time.
A natural stylized model for such temporal dependence is a Markov chain whose stationary distribution is a GGM,
for instance single-site Gibbs sampling (Glauber dynamics). If the chain mixes rapidly, then by spacing observations by at least the mixing time one can obtain
approximately independent draws and reduce to the classical i.i.d.\ setting.

However, mixing-based reductions can be unsatisfactory even for Gaussian targets.
Indeed, for a multivariate normal target, single-site Gibbs has an explicit linear-operator description, and its convergence rate is controlled by the spectrum of an associated update matrix \cite{amit1991rates,roberts1997updating}.
Moreover, this convergence behavior is invariant under diagonal rescaling of the coordinates, so it is expressed in terms of the normalized precision matrix $\Theta' = D^{-1/2}\Theta D^{-1/2}$, where $D=\mathrm{diag}(\Theta_{11},\ldots,\Theta_{nn})$.
For example, a standard spectral-gap bound for single-site Gibbs implies
\begin{equation}
\label{eq:mixing-time}
t_{\mathrm{mix}}(\varepsilon) \approx  n\,  
\frac{1}{\lambda_{\min}(\Theta')}
\,\log(1/\varepsilon),
\end{equation}
where the equality is up to absolute constants \cite{roberts1997updating,amit1991rates}. Since $\lambda_{\min}(\Theta')$ can be arbitrarily small under the~\ref{eq:non-degeneracy} constraint,\footnote{See \Cref{app:nondegeneracy-conditioning} for a simple demonstration of this.} the mixing time can be arbitrarily large.

This leads to our main question:
\begin{quote}
\emph{Is it possible to learn the structure of a sparse GGM from a single Glauber trajectory
\emph{without} waiting for the chain to mix, and without imposing additional assumptions on the precision matrix $\Theta$?}
\end{quote}

We answer this question in the affirmative by giving an efficient algorithm that recovers the graph from a single trajectory, with no dependence on the mixing time and without imposing additional assumptions on the precision matrix beyond the~\ref{eq:non-degeneracy} condition.
In the next section we formalize the model and state our main theorem.

\subsection{Results}

We begin by recalling two definitions that formalize our setting: $(\alpha,d)$-sparse Gaussian graphical models and single-site Glauber dynamics.

\begin{definition}[$(\alpha,d)$-sparse Gaussian graphical model]
\label{def:sparse-ggm}
Let $\Sigma\in\mathbb{R}^{n\times n}$ be positive definite and let $\Theta\coloneq \Sigma^{-1}$.
For $\alpha>0$ and an integer $d\ge 1$, we say that $X\sim\mathcal{N}(0,\Sigma)$ is an $(\alpha,d)$-sparse GGM
if the graph $G=(V,E)$ on $V=[n]$ defined by
\[
(i,j)\in E \quad \Longleftrightarrow \quad \Theta_{ij}\neq 0,
\]
has maximum degree at most $d$, and moreover every edge has normalized strength at least $\alpha$, i.e.,
\[
\left\lvert\frac{\Theta_{ij}}{\sqrt{\Theta_{ii}\Theta_{jj}}}\right\rvert\ \ge\ \alpha
\qquad \forall (i,j)\in E.
\]
We refer to $G$ as the (conditional-independence) graph or the sparsity pattern of the model.
\end{definition}

Next we define the Glauber dynamics for a GGM. We work with the continuous-time dynamics, where each coordinate has an independent rate-$1$ Poisson clock, so the chain performs $n$ updates per unit time in expectation.

\begin{definition}[Continuous-time Glauber dynamics for a GGM]
\label{def:cont-glauber-dynamics}
The continuous-time Glauber dynamics for a GGM with precision matrix $\Theta$
is a random process $\{Y^{(t)}\}_{t\ge 0}$ taking values in $\mathbb{R}^n$.
It is initialized at an arbitrary (possibly random or worst-case) vector $Y^{(0)}$, and is updated at random times
$\{S^{(\ell)}\}_{\ell\in\mathbb{N}}$ with $S^{(0)}=0$, where $S^{(\ell)}$ is the time of the $\ell$-th update.
The inter-update times $\{S^{(\ell+1)}-S^{(\ell)}\}_{\ell\ge 0}$ are i.i.d.\ sampled from the exponential distribution with parameter $n$, namely $\mathrm{Exp}(n)$, so the chain performs
$n$ updates per unit time in expectation.
The process is piecewise constant between updates; define the embedded discrete-time chain
$X^{(\ell)} := Y^{(t)}$ for $t\in [S^{(\ell)}, S^{(\ell+1)})$.

At each update time $S^{(\ell)}$, an index $I^{(\ell)}\in[n]$ is chosen uniformly at random.
Let $i=I^{(\ell)}$. Then we resample coordinate $i$ from its conditional distribution given the others:
\[
X_i^{(\ell)} \sim \mathcal{N}\left(-\sum_{j\in N(i)} \frac{\Theta_{ij}}{\Theta_{ii}}\, X_j^{(\ell-1)},\ \frac{1}{\Theta_{ii}}\right),
\]
and set $X_j^{(\ell)} = X_j^{(\ell-1)}$ for all $j\neq i$.
Here $N(i)=\{j\neq i:\Theta_{ij}\neq 0\}$ denotes the neighborhood of $i$ in the sparsity graph of $\Theta$.
\end{definition}

We are now ready to present our main result.

\begin{theorem}[Main Theorem (Structure learning)]
\label{thm:main-theorem-structure-learning}
There exists a polynomial-time algorithm which, given a Glauber trajectory
from an $(\alpha,d)$-sparse Gaussian graphical model on $n$ vertices,
recovers the sparsity pattern (i.e., $\mathrm{supp}(\Theta)$) with probability at least $1-\delta$, provided the trajectory length satisfies
\[
T=\Omega\!\left(\frac{d^3\log(n/\delta)}{\alpha^5}\right).
\]
\end{theorem}

A few remarks are in order.
First, the guarantee makes no mixing-time assumption and does not require any additional regularity conditions on $\Theta$
beyond sparsity and the edge-strength condition in Definition~\ref{def:sparse-ggm}.
Second, in the continuous-time dynamics of Definition~\ref{def:cont-glauber-dynamics}, the chain performs $n$ single-site
updates per unit time in expectation; thus, observing the chain up to time horizon $T$ corresponds to about $nT$
single-site updates on average.
Third (on optimality), information-theoretic lower bounds of $\Omega(\log(d) / \alpha^2)$ were previously known (see Theorem~2 in \cite{tirukkonda2025structure}).
Their statement is parameterized by $\beta_{\min}\coloneq \min_{(i,j)\in E} |\Theta_{ij}|/\Theta_{ii}$,
whereas we work with the symmetric normalization
$\min_{(i,j)\in E} |\Theta_{ij}|/\sqrt{\Theta_{ii}\Theta_{jj}}$.
Under the condition that $\Theta$ has diagonal $1$, these parameters are equal.
In \Cref{sec:info-theoretic-lower-bound} (\Cref{thm:infolb}) using similar techniques, we show an improved information-theoretic lower bound
that $\Omega(n \log(n)/\alpha^2)$ single-site updates are necessary for structure learning from a single trajectory (for success probability $1/2$). Since the number of single-site updates and the continuous-time trajectory length are equivalent up to a factor of $n$, this implies a lower bound of $\Omega(\log(n) / \alpha^2)$ on the continuous-time trajectory length.\footnote{See \Cref{app:ct-dt-reduction} for a formal statement of this equivalence.}
Thus, while our result matches the correct logarithmic dependence, its polynomial dependence on $d$ and $\alpha$ may not be optimal. Closing the remaining gap between the lower bound and the current no-mixing upper bounds is an interesting open problem.

\Cref{thm:main-theorem-structure-learning} follows from the more general parameter-learning theorem below. Our parameter-learning theorem, \Cref{thm:main-theorem-parameter-learning}, learns each coordinate up to multiplicative factor $\varepsilon$. The structure-learning theorem then follows by choosing $\varepsilon = O(1)$ and outputting an edge if $\lvert\widehat{\Theta}_{ij}\rvert > \frac{\alpha}{2}$.

\begin{theorem}[Main Theorem (Parameter learning)]
\label{thm:main-theorem-parameter-learning}
Let $0<\alpha,\delta,\varepsilon<1$ and let $n,d\in\mathbb N$.
Given a Glauber trajectory evolving according to an $(\alpha,d)$-sparse GGM with precision matrix $\Theta$ and whose length satisfies
\[
T=\Omega\!\left(\frac{d^3\log(n/\delta)}{\alpha^5\varepsilon^5}\right),
\]
there is a polynomial-time algorithm that outputs an estimate $\widehat{\Theta}$ such that
\[
\left\lvert\widehat{\Theta}_{ij}-\Theta_{ij}\right\rvert \le \varepsilon |\Theta_{ij}|
\quad\forall i,j,
\]
with success probability $1-\delta$.
\end{theorem}

We also show the following result when explicit dependence on the mixing time is allowed, without requiring any assumption on the precision matrix beyond the~\ref{eq:non-degeneracy} condition.

\begin{theorem}[Structure Learning with Mixing]
\label{thm:main-mixing-theorem}
There is a polynomial-time algorithm with the following guarantee. For some absolute constant $c>0$, given a length-$T$ Glauber trajectory from an $(\alpha,d)$-sparse Gaussian graphical model on $n$ vertices, the algorithm recovers the sparsity pattern $\mathrm{supp}(\Theta)$ with probability at least $1-\delta$, provided the trajectory length satisfies
\[
T = \Omega\!\left(
\frac{\log(n/\delta)}{\alpha^2}
\left(
t_{\mathrm{mix}}(c\alpha)+\frac{d^2}{\alpha^2}
\right)
\right).
\]
\end{theorem}

The proof uses the shorter ``$iji$'' pattern together with mixing gaps between accepted windows. On a clean nearly stationary window, the edge signal appears as the regression coefficient between the $j$-increment and the later $i$-increment, so after conditioning on a large $j$-increment the resulting ratio statistic is a bounded-variance estimate of \(\Theta'_{ij}\).

In particular, when the mixing time is guaranteed to be small, for example \(t_{\mathrm{mix}}(c\alpha)=O(\log(1/\alpha))\), the bound becomes \(T = O\!\left(\frac{d^2\log(n/\delta)}{\alpha^4}\right)\).
This improves the bound in \Cref{thm:main-theorem-structure-learning} by a factor of \(d/\alpha\). Moreover, relative to the i.i.d.\ setting, the algorithm effectively requires \(\log(n/\delta)/\alpha^2\) nearly independent samples, matching the lower bound for i.i.d.\ data shown in \cite{wang2010information}.

Similar results with mixing were obtained previously in \cite{tirukkonda2025structure}. Compared with that work, \Cref{thm:main-mixing-theorem} gives an improved dependence on \(\log(1/\delta)\) and does not require any additional assumptions on the precision matrix. We discuss the extra assumptions in \cite{tirukkonda2025structure} in \Cref{sec:related-work}.

Finally we discuss the practicality of our algorithms.
\begin{remark}[On Practicality]
Despite our asymptotic improvements in trajectory length,
constants that appear in our analysis preclude the practicality of our algorithm.
We believe that these constants are not significantly improvable,
and indeed, our experiments suggest that the trajectory lengths required by our algorithms are too large to be practical.
We briefly explain the theoretical necessity for these large constants
at the end of \Cref{sec:main-algorithm}. Obtaining practical algorithms for this problem is an interesting direction for future work.
\end{remark}

To see a full proof of \Cref{thm:main-theorem-structure-learning,thm:main-theorem-parameter-learning} see \Cref{sec:main-algorithm}. For a full proof of \Cref{thm:main-mixing-theorem}, see \Cref{sec:mixed_ub}. For a technical overview of the algorithm see \Cref{sec:technical-overview}.

\subsection{Related work}
\label{sec:related-work}

\paragraph{Gaussian GGMs from Glauber dynamics.}
Most closely related to our work is the recent work of Tirukkonda, Rayas, and Dasarathy \cite{tirukkonda2025structure},
which gives the first algorithmic guarantees for structure learning in Gaussian graphical models from a single Glauber trajectory,
along with information-theoretic lower bounds.
In \Cref{sec:serious-gap} we discuss a technical issue in the proof of one of their lemmas.
{
After we notified the authors of this gap and posted a preliminary version of this paper on the MIT SPUR website\footnote{\href{https://math.mit.edu/documents/spur/2025Shen_Wu.pdf}{https://math.mit.edu/documents/spur/2025Shen\_Wu.pdf}}, they addressed the issue caused by looking at ``$iji$'' updates by adopting an approach similar to ours, based on ``$iiji$" updates. However, both the earlier version and the updated version of their manuscript still rely crucially on several regularity assumptions beyond a minimum edge-strength assumption.
}

Concretely, in addition to requiring $|\Theta_{ij}|/\Theta_{ii}\ge \beta_{\min}$ for all $(i,j)\in E$, they assume
(i) an upper bound $|\Theta_{ij}|/\Theta_{ii}\le \beta_{\max}$ for all $(i,j)\in E$ (Assumption~A1),
(ii) bounded marginals $\Sigma_{ii}\le \sigma_{\max}^2$ and non-degenerate conditionals $1/\Theta_{ii}\ge \sigma_{\min}^2$ (Assumption~A2),
and (iii) a sample decay condition $d\,\beta_{\max}<1-\Omega(1)$ (Assumption~A3).\footnote{As stated it is written as $d\,\beta_{\max}<1$; however, the proof of Lemma~4 explicitly uses that this gap is at least a constant.}
Meanwhile, we make no additional assumptions beyond the minimum edge-strength assumption.
Most importantly, in the diagonal-$1$ and $d=O(1)$ setting, the last assumption implies very fast mixing via the Gershgorin circle theorem:
$\lambda_{\min}(\Theta) \ge 1-d\beta_{\max} = \Omega(1)$ and $\lambda_{\max}(\Theta) \le 1+d\beta_{\max}=O(1)$ (see \Cref{eq:mixing-time}).
Therefore, under their assumptions one can essentially simulate approximate i.i.d.\ samples by observing the Glauber trajectory at constant time intervals.
However, in our setting, the mixing time may be arbitrarily slow. 
{
Under constant mixing-time bounds and polynomial runtime, \Cref{thm:main-mixing-theorem} improves the dependence on $d$ and $\alpha$ by a factor of $d / \alpha$ compared with their algorithm. In addition, our result avoids the extra assumptions A1 and A2 required in their analysis.
}


\newcommand{\cmark}{\checkmark}
\newcommand{\xmark}{\ensuremath{\times}}

\begin{table}[t]
\scriptsize
\renewcommand{\arraystretch}{1.45}
\centering
\label{tab:theorem-comparison}
\renewcommand{\arraystretch}{1.2}
\begin{tabular}{@{}l c c c c@{}}
\toprule
\textbf{Result}
& \textbf{Trajectory length}
& \textbf{Polynomial runtime}
& \textbf{No mixing}
& \textbf{Extra assumptions} \\
\midrule

\cite{tirukkonda2025structure} Theorem~1
&
$O \left(d^3\cdot \left(\dfrac{\log^{^{6/5}} (n/\delta)\sigma^2_{\max}}{(1-d\beta_{\max})^2\sigma^2_{\min}\beta_{\min}}\right)^5 \right)$
&
\cmark
&
\xmark
&
\cite{tirukkonda2025structure}'s A[1-3]

\\
\addlinespace[0.5em]

\cite{tirukkonda2025structure} Theorem~2
&
$\tilde{O}\left(\dfrac{d \log^2(n/\delta)}{(1-d \beta_{\max}) \alpha^2}\right)$
&
\xmark
&
\xmark
&
\cite{tirukkonda2025structure}'s A[1-3]

\\
\addlinespace[0.5em]

\Cref{thm:main-mixing-theorem}
&
$O\left(
\dfrac{\log(n/\delta)}{\alpha^2}
\left(
t_{\mathrm{mix}}(\alpha)+\dfrac{d^2}{\alpha^2}
\right)
\right)$
&
\cmark
&
\xmark
&
None

\\
\addlinespace[0.5em]

\Cref{thm:main-theorem-structure-learning}
&
$O\left(\dfrac{d^3 \log(n/\delta)}{\alpha^5}\right)$
&
\cmark
&
\cmark
&
None

\\
\bottomrule
\end{tabular}
\caption{Comparison of structure learning results. In the notation of \cite{tirukkonda2025structure}, $\beta_{\max}, \beta_{\min}$ denote the largest and the smallest off diagonal value in the precision matrix. $\sigma^2_{\max}$ is the upper bound on the diagonal entries on the covariance matrix and $\sigma^2_{\min}$ is the lower bound on $1/\Theta_{ii}$. $\tilde{O}$ hides lower order logarithmic factors in all parameters.}
\end{table}

\paragraph{Learning graphical models from Glauber dynamics.}
A growing line of work studies structure learning when observations arise from local Markov dynamics rather than i.i.d.\ sampling.
Bresler, Gamarnik, and Shah \cite{bresler2017learning} initiated this direction for discrete graphical models,
showing that observing a single-site Glauber trajectory can make structure learning computationally tractable for Ising models.
More recently, Gaitonde and Mossel \cite{gaitonde2024unified} gave a unified analysis for learning Ising models from Glauber trajectories, and Gaitonde, Moitra, and Mossel \cite{gaitonde2025better} gave the first efficient algorithms in the observation model that reveals only actual configuration changes rather than all update attempts, with extensions to reversible single-site chains such as Metropolis.
In a different direction, Gaitonde, Moitra, and Mossel \cite{gaitonde2025bypassing} show that for higher-order Markov random fields, access to Glauber dynamics trajectories can bypass i.i.d.\ hardness barriers (e.g., noisy parity), yielding efficient learning.

Technically, our algorithm is close in spirit to several methods in this line of work: we examine short windows of the trajectory to probe
the interaction between a candidate pair $(i,j)$, and we exploit sparsity to ensure that with sufficiently large probability no neighbor of
$i$ or $j$ updates within the window, allowing the local effect of $(i,j)$ to be isolated from confounding updates. That said, structure learning for GGMs differs in important ways from the Ising setting.
First, the variables are continuous, so one cannot rely on discrete indicator statistics present in previous work whose empirical averages directly estimate event probabilities.
Second, in the Gaussian case a key difficulty is \emph{anti-concentration}: to detect an edge $(i,j)$ one needs the neighbor's influence on the conditional mean to be typically non-negligible.
Unlike the Ising setting—where boundedness and discrete concentration can often be leveraged—Gaussian anti-concentration depends on the scale of the conditional variance (equivalently, on $\Theta_{ii}$) and can degrade without additional regularity beyond $(\alpha,d)$.
As a result, while the high-level philosophy is shared, the Gaussian case requires new technical ideas.

\section{Technical Overview}
\label{sec:technical-overview}

At a high level, our algorithm has three ingredients.
First, we estimate the diagonal entries $\Theta_{ii}$ from short windows of the trajectory and use them to normalize the model so that the precision matrix has (approximately) unit diagonal.
Second, for each candidate edge $(i,j)$, we look for short update patterns that isolate the influence of $j$ on a later update of $i$.
Third, because we cannot directly tell whether hidden neighbor updates occurred inside a window, we treat such windows as contaminated and aggregate many windows using robust median-based estimators.

We begin with diagonal estimation because it already contains the main ideas of the full algorithm: short informative windows, unobserved contamination, and robust aggregation.
We then explain how to estimate off-diagonal entries, and finally describe a simpler variant that is available when the chain is allowed to mix between samples.

\paragraph{Two nearby $i$-updates reveal $\Theta_{ii}$.}
Suppose that at some point in the trajectory the $i$th coordinate updates twice, with no update to any neighbor of $i$ in between.
Write the corresponding states as
\[
X^{(0)}\stackrel{i}\longrightarrow X^{(1)}\stackrel{i}\longrightarrow X^{(2)}.
\]
Since no neighbor of $i$ changes between the two $i$-updates, both updates use the same conditional mean.
Hence
\[
X_i^{(1)}-X_i^{(2)}\sim \mathcal{N}\!\left(0,\frac{2}{\Theta_{ii}}\right).
\]
So every such window gives a sample whose variance is exactly $2/\Theta_{ii}$.
This turns diagonal estimation into a one-dimensional robust variance-estimation problem.

\paragraph{The ``ii'' pattern and hidden contamination.}
Exact consecutive ``ii'' updates are too rare to be useful on the trajectory lengths we target.
Instead, we divide the trajectory into short windows of length $T$ and keep every window that contains at least two updates to coordinate $i$.
A window is \emph{clean} if, in addition, no neighbor of $i$ updates inside that window.
On a clean window, the same calculation as above shows that the difference between the two updated values of coordinate $i$ is distributed as $\mathcal{N}(0,2/\Theta_{ii})$.

The difficulty is that cleanliness depends on the unknown neighborhood of $i$, so we cannot test it directly.
We therefore keep all windows with two $i$-updates and view the non-clean ones as contaminated samples.
Because the window is short and the graph is $d$-sparse, the contamination rate is small.
Taking the median absolute deviation of the resulting samples gives a robust estimate of $\Theta_{ii}$.

\paragraph{Normalization.}
Using a small initial portion of the trajectory, we estimate every diagonal entry $\Theta_{ii}$.
We then rescale coordinate $i$ by $\sqrt{\Theta_{ii}}$ (or, in the actual algorithm, by its estimate), replacing each state $X\in\mathbb R^n$ with $X'$ defined by
\(X_i' = \sqrt{\Theta_{ii}}\,X_i\).
Under exact normalization, $X'$ is again a Glauber trajectory, now for a precision matrix with diagonal entries equal to $1$.
With estimated diagonals, the normalized trajectory has diagonal entries within $1\pm \varepsilon$, and the later analysis is stable to this approximation.
For the overview, we therefore assume from this point onward that the diagonal is exactly $1$.

\paragraph{A naive ``iji'' test for $\Theta_{ij}$.}
To estimate an off-diagonal entry $\Theta_{ij}$, the most natural idea is to look for a short window of the form
\[
X^{(1)}\stackrel i\longrightarrow X^{(2)}
\stackrel j\longrightarrow X^{(3)}
\stackrel i\longrightarrow X^{(4)},
\]
again with no neighbor of $i$ or $j$ updating inside the window.
On such a clean window,
\begin{align*}
X_i^{(2)} &= \sum_{k\ne i} -\Theta_{ik}X_k^{(1)} + \varepsilon^{(2)},\\
X_i^{(4)} &= \sum_{k\ne i} -\Theta_{ik}X_k^{(3)} + \varepsilon^{(4)},
\end{align*}
with $\varepsilon^{(2)},\varepsilon^{(4)}\sim \mathcal{N}(0,1)$ i.i.d.
The only relevant change between the two conditional means is the value of coordinate $j$, so
\[
X_i^{(4)}-X_i^{(2)}
= -\Theta_{ij}\bigl(X_j^{(3)}-X_j^{(1)}\bigr)
+ \varepsilon^{(4)}-\varepsilon^{(2)}.
\]
This suggests estimating $\Theta_{ij}$ from the ratio
\[
\frac{X_i^{(4)}-X_i^{(2)}}{X_j^{(3)}-X_j^{(1)}}.
\]
To keep the denominator well behaved, we further condition on \(\lvert X_j^{(3)}-X_j^{(1)}\rvert>c\) for a fixed constant $c>0$.

\paragraph{Why the ``iji'' test fails.}
The problem is a subtle dependence issue.
In the ``iji'' pattern, the middle $j$-update depends on the value produced by the first $i$-update.
As a result, the noise term from the first $i$-update is not independent of the denominator $X_j^{(3)}-X_j^{(1)}$.
So the ratio above is not centered in the simple way suggested by the heuristic calculation.
This is the main obstacle that forces us away from the naive ``iji'' statistic.

\paragraph{The ``iiji'' pattern breaks the dependence.}
To remove this dependence, we prepend one extra $i$-update and instead search for windows of the form
\[
X^{(0)}\stackrel i\longrightarrow X^{(1)}\stackrel i\longrightarrow
X^{(2)}\stackrel j\longrightarrow X^{(3)}\stackrel i\longrightarrow X^{(4)},
\]
where the visible event requires that $i$ updates in the first, second, and
fourth quarters, and that $j$ updates in the third quarter, with no off-pattern
$i/j$ updates. A window is clean if, in addition, no coordinate in
$(N(i)\cup N(j))\setminus\{i,j\}$ updates in the window. Conditioned on
$X^{(0)}$, the extra $i$-update acts as a refresh step: the second $i$-update is
independent of the first, so the noise introduced at $X^{(1)}$ is independent of
the later state of coordinate $j$.
This yields the relation
\[
X_i^{(4)}-X_i^{(1)}
= -\Theta_{ij}\bigl(X_j^{(3)}-X_j^{(0)}\bigr)
+ \varepsilon^{(4)}-\varepsilon^{(1)},
\]
where $\varepsilon^{(1)},\varepsilon^{(4)}\sim \mathcal{N}(0,1)$ are independent.
After conditioning on \(\bigl\lvert X_j^{(3)}-X_j^{(0)}\bigr\rvert>c\),
the ratio statistic
\[
\frac{X_i^{(4)}-X_i^{(1)}}{X_j^{(3)}-X_j^{(0)}}
= -\Theta_{ij} + \frac{\varepsilon^{(4)}-\varepsilon^{(1)}}{X_j^{(3)}-X_j^{(0)}},
\]
is centered at $-\Theta_{ij}$ and has variance bounded by a constant.
Thus each clean ``iiji'' window gives a noisy but informative estimate of $\Theta_{ij}$.

\paragraph{Robust aggregation.}
As in the diagonal-estimation step, we cannot directly verify that a window is
clean, because the unknown graph determines which coordinates count as hidden
neighbors. We therefore select windows using only the visible ``iiji'' pattern
and treat windows with hidden neighbor updates as contaminated.
Moreover, the variance of the ratio statistic can vary from one accepted window to another, and successive windows are not independent because they come from a single trajectory.
Nevertheless, the dependence is structured enough that martingale concentration can be used in place of the usual independent-sample Chernoff argument.
This shows that the sample median remains a robust estimator of the common location parameter $-\Theta_{ij}$.
Estimating every $\Theta_{ij}$ to additive error $O(\alpha)$ is then enough to recover the graph structure.
With tighter parameter settings, the same framework also yields multiplicative estimation of the full precision matrix.

\paragraph{Learning with mixing.}
We also give a simpler and more sample-efficient algorithm when the mixing time is allowed to enter the bound.
We insert waiting periods of length about $t_{\mathrm{mix}}$ between accepted
windows so that each accepted window begins from an almost stationary state.
In this regime, the shorter ``iji'' pattern suffices. On a clean stationary
window of the normalized trajectory,
\[
X'^{(0)}\stackrel i\longrightarrow X'^{(1)}\stackrel j\longrightarrow
X'^{(2)}\stackrel i\longrightarrow X'^{(3)},
\]
the pair
\[
\left(
X'^{(2)}_j-X'^{(0)}_j,\;
X'^{(3)}_i-X'^{(1)}_i
\right)
\]
is centered Gaussian with covariance matrix
\[
\begin{pmatrix}
2 & -\Theta'_{ij}\\
-\Theta'_{ij} & 2
\end{pmatrix}.
\]
Thus the edge signal appears in the covariance, equivalently in the regression
coefficient of the later $i$-increment on the earlier $j$-increment. After
passing to the observable trajectory and conditioning on a large
$j$-increment, the ratio statistic
\(
-2\widehat\Delta_i/\widehat\Delta_j
\)
is a bounded-variance noisy estimate of $\Theta'_{ij}$. A robust median over
well-separated epochs then recovers the edge, improving the dependence on $d$
and $\alpha$ when mixing is available.

\paragraph{Information-theoretic lower bound.}
Finally, we prove an information-theoretic lower bound showing that logarithmic trajectory length is unavoidable even without computational constraints.
We construct a family of GGMs on $2n$ vertices whose graphs are disjoint unions of edges, with each candidate obtained by deleting one edge from a perfect matching.
The resulting trajectories are hard to distinguish from one another.
By upper-bounding the pairwise KL divergence and applying Fano's inequality, we obtain a lower bound on the number of Glauber updates required for structure learning.

Our lower bound also differs from \cite{tirukkonda2025structure} in both the hard family and the KL calculation. Instead of working with a broader graph-class construction, we use a simple family obtained from a perfect matching by deleting a single edge. More importantly, we work directly with the exact KL divergence of the Gaussian conditional-update distributions, rather than conditioning on small updates and comparing truncated Gaussians. This yields a $\log n$ improvement in the lower bound. Closing the remaining gap to our current no-mixing upper bounds is an interesting open problem.

\section{Preliminaries}
\label{sec:preliminaries}

\subsection{Continuous-time Glauber dynamics}
We consider Glauber dynamics
in the continuous-time setting.
The initial configuration is arbitrary,
and each coordinate $i=1,\ldots,n$ updates according
to an independent Poisson process of rate 1.
The following lemma records the relevant update probabilities:
\begin{lemma}[Lemma~3.1 in \cite{gaitonde2025bypassing}]\label{lem:poissonbounds}
Given an interval $I\subset\mathbb R_{\ge0}$ of length $T$,
let $U_I$ denote the set of coordinates that update in $I$.
For any $S\subset[n]$ with $|S|=\ell$, we have
\begin{align*}
    \mathbb{P}[S\cap U_I=\varnothing]&=\exp(-T\ell),\\
    \text{and}\quad
    \mathbb{P}[S\subseteq U_I]&\ge1-\ell\exp(-T).
\end{align*}
\end{lemma}
\begin{proof}
    For each Poisson process $\Pi$ of rate 1
    and interval $I\subseteq\mathbb R_{\ge0}$,
    we have $\mathbb{P}[\Pi\cap I=\varnothing]=\exp(-|I|)$,
    and the claimed bounds follow immediately.
\end{proof}

\subsection{Observing patterns}
Throughout the paper, we look for intervals
containing a specific ``pattern,'' which we concretely define below:
\begin{definition}[Pattern exhibition and strictness]
\label{def:pattern}
    A \emph{pattern} of length $k$
    is a sequence of indices $P=(i_1,i_2,\ldots,i_k)\in\{1,\ldots,n\}^k$.
    Consider a pattern $P$
    and an interval $I=[t,t+T)$ of a continuous-time single-site update Glauber trajectory.
    Let $I_j=[t+(j-1)T/k,t+jT/k)$ for $j=1,\ldots,k$.
    We say $I$ \emph{exhibits} pattern $P$
    if index $i_j$ updates in $I_j$ for each $j=1,\ldots,k$.
    We say $I$ is \emph{strict} with respect to $P$
    if, for each $j=1,\ldots,k$, no neighbor of any vertex appearing in $P$,
    other than $i_j$, updates in $I_j$.
\end{definition}

We next bound the probabilities that a given interval exhibits a pattern
and is strict with respect to it.
\begin{lemma}\label{lem:strict}
    Let $I$ be an interval of length $T$,
    and let $P$ be a pattern containing $\ell$ distinct indices.
    Then with probability at least $\exp(-T\ell d)$, $I$ is strict with respect to $P$.
\end{lemma}
\begin{proof}
    Each index in $P$ has at most $d$ neighbors,
    so we are forbidding at most $\ell d$ neighbors from updating
    in each $I_j$. By \Cref{lem:poissonbounds},
    each $I_j$ avoids its respective indices
    with probability at least $\exp(-T/k\cdot\ell d)$.

    The events in the $k$ subintervals are independent,
    so the probability this holds for all $k$ subintervals
    is at least $\exp(-T\ell d)$.
\end{proof}
\begin{lemma}\label{lem:exhibit}
    Let $I$ be an interval of length $T\le1/3$,
    and let $P$ be a pattern of length $k$.
    Then with probability at least $T^k/(2k^k)$, $I$ exhibits $P$.
\end{lemma}
\begin{proof}
    By \Cref{lem:poissonbounds},
    \begin{align*}
    \mathbb{P}[I\text{ exhibits }P]
    &\ge(1-\exp(-T/k))^k\\
    &\ge\frac{T^k}{k^k}\exp(-T)
    \ge\frac{T^k}{2k^k},
    \end{align*}
    where we use the estimates $1-\exp(-x)\ge x\exp(-x)$ and $T\le1/3$.
\end{proof}
\begin{lemma}\label{lem:pattern-independence}
    For any interval $I$ and pattern $P$, the event that $I$ exhibits $P$ is
    independent of the event that $I$ is strict with respect to $P$.
\end{lemma}
\begin{proof}
    In each subinterval \(I_j\), exhibition depends only on the clock of
    \(i_j\), while strictness depends only on the clocks of the remaining
    coordinates. Thus the two events are independent in each \(I_j\), and the
    claim follows by independence across the disjoint subintervals.
\end{proof}

\subsection{Corruption and robust estimation}
We use several robust estimators in this paper.
We work with the following corruption model.
\begin{definition}[$\eta$-corruption]
\label{def:corruption-model}
Let $\mathcal{D}$ be a distribution, and let
$X_1,\dots,X_n \sim \mathcal{D}$ be i.i.d.\ samples.
We say that the observed samples
$\widetilde X_1,\ldots,\widetilde X_n$ are
\emph{$\eta$-corrupted} if there exist i.i.d.\ $\text{Bernoulli}(\eta)$ random variables
$\xi_1,\ldots,\xi_n$ such that for each $i=1,\ldots,n$,
\[
\widetilde X_i =
\begin{cases}
X_i, & \text{if } \xi_i = 0 \\
A_i, & \text{if } \xi_i = 1,
\end{cases}
\]
where each $A_i$ is an arbitrary value.
Whenever $\widetilde X_i=A_i$, the corrupted value may be chosen adversarially, may depend on the entire trajectory up to the corrupted sample $\widetilde X_i$, and need not be drawn from any distribution.\footnote{This corruption model is closely related to what some recent work calls \emph{malicious noise} or \emph{strong malicious noise}; see, e.g., \cite{blanc2026nasty}.}
\end{definition}

Note that our corruption model differs from the Huber contamination model, as corrupted entries may be selected adversarially.
It also differs from the strong contamination model
in that each sample is independently corrupted with probability $\eta$.
Further, as Glauber trajectories are sampled sequentially,
corrupted samples in our corruption model may only depend on prior samples.

Under this corruption model,
we use the following robust estimator for variance.

\begin{restatable}[Robust variance estimation]{lemma}{robustvarianceestimator}
\label{lem:robustvarest}
Let $0<\eta\le1/10$ and $0<\delta<1$.
There is a linear-time algorithm that takes
$n$ samples from $\mathcal{N}(0,\sigma^2)$,
an $\eta$-fraction of which are corrupted (as in \Cref{def:corruption-model}),
and outputs $\widehat{\sigma}$ such that
\(\lvert\widehat\sigma-\sigma\rvert<5\eta\sigma\)
with probability $1-\delta$, provided \(n \ge \frac{2\log(4/\delta)}{\eta^2}\).
\end{restatable}

We also use the following mean estimator,
which allows for another source of adversarial control
in the choice of the variances.
\begin{restatable}[Robust mean estimation]{lemma}{robustmeanestimator}
\label{lem:advrobustmeanest}
    Let $0<\eta\le1/10$ and
    $0<\delta<1$.
    Let $\Phi$ denote the cdf of a standard Gaussian.
    We are given a filtration $(\mathcal F_\ell)_{\ell=0}^n$ and samples
    $x^{(1)},\ldots,x^{(n)}$.
    Suppose there exist indicators $\xi^{(\ell)}\in\{0,1\}$ such that
    \[
        \mathbb E[\xi^{(\ell)}\mid \mathcal F_{\ell-1}] \le \eta
        \qquad\text{for all }\ell.
    \]
    Assume moreover that whenever $\xi^{(\ell)}=0$, the clean sample is centered
    at the same value $\mu$ and has Gaussian tails dominated by unit variance in
    the sense that, for every $t\ge 0$,
    \[
        \mathbb P\!\left[
            x^{(\ell)}\ge \mu+t
            \,\middle|\, \mathcal F_{\ell-1},\xi^{(\ell)}=0
        \right]
        \le 1-\Phi(t),
    \]
    \[
        \mathbb P\!\left[
            x^{(\ell)}\le \mu-t
            \,\middle|\, \mathcal F_{\ell-1},\xi^{(\ell)}=0
        \right]
        \le 1-\Phi(t).
    \]
    Then the sample median
    $\widehat\mu$ satisfies
    \(\lvert\widehat\mu-\mu\rvert<5\eta\)
    with probability $1-\delta$, provided \(n\ge\frac{2\log(2/\delta)}{\eta^2}\).
\end{restatable}
Note that \Cref{lem:robustvarest} remains valid even when the corrupted samples are chosen adversarially with knowledge of the full trajectory of iterates, whereas \Cref{lem:advrobustmeanest} does not.
In particular, \Cref{lem:advrobustmeanest} applies whenever, conditional on $\mathcal F_{\ell-1}$ and on the sample being uncorrupted, the clean law is a Gaussian scale mixture $\mu+\sigma\varepsilon$ with $|\sigma|<1$: averaging over $\sigma$ preserves the displayed tail bounds.
 
The proofs of these estimators are provided in \Cref{sec:robust-estimates}.

\subsection{Concentration inequalities}
We use the following concentration inequalities in the proofs of
\Cref{lem:robustvarest,lem:advrobustmeanest}.
\begin{fact}[Dvoretzky--Kiefer--Wolfowitz]\label{fact:DKW}
Let \(Y_1,\ldots,Y_n\) be i.i.d.\ real-valued random variables with cdf \(F\), and let
\[
F_n(t):=\frac1n\sum_{i=1}^n \mathbf{1}\{Y_i\le t\},
\]
be the empirical cdf. Then for every \(\varepsilon>0\),
\[
\mathbb{P}\!\left[\sup_{t\in\mathbb R}\bigl|F_n(t)-F(t)\bigr|>\varepsilon\right]
\le 2e^{-2n\varepsilon^2}.
\]
\end{fact}

\begin{fact}[Azuma--Hoeffding]\label{fact:azuma-hoeffding}
Let \((M_k,\mathcal F_k)_{k=0}^n\) be a martingale, and suppose that for each
\(k=1,\ldots,n\),
\[
M_k-M_{k-1}\in [a_k,b_k] \qquad \text{almost surely}.
\]
Then for every \(t>0\),
\[
\mathbb{P}[M_n-M_0\ge t]
\le
\exp\!\left(
-\frac{2t^2}{\sum_{k=1}^n (b_k-a_k)^2}
\right).
\]
In particular, if each increment lies in an interval of length at most \(1\), then
\[
\mathbb{P}[M_n-M_0\ge t]\le e^{-2t^2/n}.
\]
\end{fact}

\subsection{Mixing and coupling}
\begin{definition}[Mixing and mixed chains]
\label{def:mixing-and-mixed-chains}
    A position $X_i$ of the Markov chain $X_0$, $X_1$, $\ldots$ is \emph{$\veps$-mixed}
    if $X_i$ has TV-distance at most $\veps$ from its stationary distribution $\pi$.
    We define the \emph{mixing time} by
    \begin{align*}
    t_{\mathrm{mix}}(\varepsilon)
    &=\min\Big\{t\ge0
    :\max_x
    \TV(\mathcal L(X_t\mid X_0=x),\pi)\le\varepsilon\Big\},
    \end{align*}
    where the maximum is over all starting states $x$.
\end{definition}

The stationary distribution of a Glauber trajectory with
covariance matrix $\Sigma$ is $\mathcal{N}(0,\Sigma)$.
We use the following slightly stronger invariance statement.
\begin{lemma}
\label{lem:stationary}
Fix a coordinate $i$.
If $x\sim \mathcal{N}(0,\Sigma)$, then after a single Glauber update
to the $i$th coordinate,
the resulting position $x'$ still satisfies
$x'\sim \mathcal{N}(0,\Sigma)$.
\end{lemma}
\begin{proof}
    Without loss of generality, we update the first coordinate.
    Write precision and covariance in block matrix form as
    \begin{align*}
    \Sigma&=\begin{pmatrix}
        \Sigma_{11}&\Sigma_{1,-1}\\
        \Sigma_{-1,1}&\Sigma_{-1,-1} 
    \end{pmatrix},\\
    \Theta&=\begin{pmatrix}
        \Theta_{11}&\Theta_{1,-1}\\
        \Theta_{-1,1}&\Theta_{-1,-1}
    \end{pmatrix}.
    \end{align*}
    The update rule is
    \[
    x'_{-1}=x_{-1},\qquad
    x_1'=-\frac{\Theta_{1,-1}}{\Theta_{11}}x_{-1}
    +\mathcal{N}\left(0,\frac1{\Theta_{11}}\right).
    \]
    Then $(x_1',x_{-1}')$ is still Gaussian
    with mean $0$. Moreover,
    \begin{align*}
        \operatorname{Cov}(x_1',x_{-1}')
        &=-\frac{\Theta_{1,-1}}{\Theta_{11}}\Sigma_{-1,-1}
        =\Sigma_{1,-1},\\
        \operatorname{Var}(x_1')
        &=\frac1{\Theta_{11}}
        +\frac{\Theta_{1,-1}}{\Theta_{11}}\Sigma_{-1,-1}
        \frac{\Theta_{-1,1}}{\Theta_{11}}
        =\Sigma_{11},
    \end{align*}
    by Schur's complement.
    Since $x'_{-1}=x_{-1}$, this proves $x'\sim\mathcal{N}(0,\Sigma)$.
 \end{proof}

From here, we may use the coupling lemma
to derive a similar statement about mixing.
\begin{fact}[Coupling lemma]\label{fact:truecoupling}
    Let $\mu$ and $\eta$ be distributions over $\mathbb R^n$.
    For any coupling $\omega$ of $\mu$ and $\eta$, if $(X,Y)\sim\omega$, then
    \(\mathbb{P}[X\ne Y]\ge\TV(\mu,\eta)\).
    Moreover, equality is attained for some coupling.
\end{fact}
\begin{corollary}\label{cor:coupling}
Let $\mu$ and $\eta$ be distributions.
Suppose $f(X)\sim\eta$ whenever $X\sim\mu$.
Then, if $Y$ is sampled from a distribution
with TV-distance $\varepsilon$ from $\mu$,
then $f(Y)$ follows a distribution
with TV-distance at most $\varepsilon$ from $\eta$.
\end{corollary}
\begin{proof}
    Let $\nu$ denote the distribution of $Y$.
    By \Cref{fact:truecoupling}, there exists a coupling of $X\sim\mu$ and $Y\sim\nu$
    such that $\mathbb{P}[X\ne Y]=\TV(\mu,\nu)\le\varepsilon$.
    Under this coupling, $f(X)\sim\eta$ and
    \(\mathbb{P}[f(X)\ne f(Y)]\le\varepsilon\).
    Another application of \Cref{fact:truecoupling} gives the claim.
\end{proof}
\begin{corollary}
\label{lem:mixingpersists}
    Fix a coordinate $i$.
    If a Glauber trajectory is $\varepsilon$-mixed (as in \Cref{def:mixing-and-mixed-chains}),
    and undergoes a Glauber update to the $i$th coordinate,
    it remains $\varepsilon$-mixed.
\end{corollary}
\begin{proof}
    Combine \Cref{lem:stationary} and \Cref{cor:coupling}.
\end{proof}

\section{Reduction to Normalized Gaussian Graphical Models}
\label{sec:normalize_diag}
In this section we design an algorithm that
reduces the problem of learning an arbitrary
sparse Gaussian graphical model
from a Glauber trajectory
to the case where the diagonal entries of the 
precision matrix are 1.

To this end, we first estimate the diagonal entries
up to a multiplicative factor $(1\pm O(\varepsilon))$.
We then scale the Glauber trajectory
$X^{(1)}$, $X^{(2)}$, $\ldots$
by \(X_i'^{(k)}:=\sqrt{\Theta_{ii}}\cdot X_i^{(k)}\) for all \(i,k\),
which we show is itself a Glauber trajectory.
Using our estimates for $\Theta_{ii}$,
we then have a coordinate-wise
$(1\pm O(\varepsilon))$-approximation
of this scaled trajectory,
which also happens to be a Glauber trajectory itself.

Our main result for this section is as follows:
\begin{theorem}[Normalization]\label{thm:normalization}
    Let $0<\alpha,\delta<1$ and $0<\varepsilon\le1/10$,
    and let $n,d\in\mathbb N$.
    Given a Glauber trajectory evolving according to an $(\alpha,d)$-sparse
    GGM with precision matrix $\Theta$ and having length
    \[
        T=\frac{4000d\log(8n/\delta)}{\varepsilon^3}+T_{\mathrm{rest}},
    \]
    where $T_{\mathrm{rest}}\ge0$,
    there is a polynomial-time algorithm
    that outputs
    an estimate $D$
    for $\operatorname{diag}(\Theta)$ such that
    \[
        \left\lvert
        \frac1{\sqrt{D_i}}-\frac1{\sqrt{\Theta_{ii}}}
        \right\rvert
        \le
        \frac{\varepsilon}{\sqrt{\Theta_{ii}}}
        \quad\forall i,
    \]
    and
    a trajectory of length $T_{\mathrm{rest}}$
    evolving according to a GGM
    with precision matrix
    $\widehat\Theta=D^{-1/2}\Theta D^{-1/2}$
    with probability $1-\delta$.
\end{theorem}

We proceed in three steps. \Cref{ss:propertystat3} identifies the ``$ii$'' statistic underlying diagonal estimation. \Cref{ss:retrieving-diagonal} uses this statistic together with robust variance estimation to prove \Cref{lem:singlethetaii}. \Cref{ss:normalizing} shows that the rescaled trajectory is again a Glauber trajectory and then completes the proof of \Cref{thm:normalization}.
\begin{lemma}\label{lem:singlethetaii}
    For fixed $i$, given a Glauber trajectory whose length satisfies
    \[T\ge\frac{4000d\log(8/\delta)}{\varepsilon^3},\] we may retrieve in polynomial
    time an estimate $D_i$ for $\Theta_{ii}$ with
    \[\left\lvert
    \frac1{\sqrt{D_i}}-\frac1{\sqrt{\Theta_{ii}}}
    \right\rvert\le\frac\varepsilon{\sqrt{\Theta_{ii}}}\]
    with probability $1-\delta$.
\end{lemma}

\subsection{Properties of the statistic}
\label{ss:propertystat3}
Fix $i$.
Let $T\le1/3$, and let $I_t$ denote the time interval $[t,t+T)$.
Let $\mathcal A^{(t)}$ denote the event that $I_t$
is strict with respect to $(i,i)$,
and $\mathcal B^{(t)}$ the event $I_t$ exhibits $(i,i)$.

If $\mathcal A^{(t)}$ and $\mathcal B^{(t)}$ both hold, 
let $Y^{(0)}$, $Y^{(1)}$, $Y^{(2)}$
denote the position at times $t$, $t+T/2$, $t+T$
respectively.

\begin{lemma} \label{iiStatisticFormula}
    We have
    \(Y^{(1)}_i-Y^{(2)}_i\mid\mathcal A^{(t)},\mathcal B^{(t)}
    \sim \mathcal{N}\left(0,\frac2{\Theta_{ii}}\right)\).
\end{lemma}
\begin{proof}
    By \Cref{def:cont-glauber-dynamics}, we have
    \begin{align*}
    Y^{(1)}_i&=-\sum_{j\ne i}-\frac{\Theta_{ij}}{\Theta_{ii}}Y^{(0)}_j
    +\varepsilon_1,\\
    \text{and}\quad
        Y^{(2)}_i&=-\sum_{j\ne i}-\frac{\Theta_{ij}}{\Theta_{ii}}Y^{(1)}_j
    +\varepsilon_2,
    \end{align*}
    where $\varepsilon_1,\varepsilon_2\sim \mathcal{N}(0,\frac1{\Theta_{ii}})$ are i.i.d.
    
    If $\mathcal A^{(t)}$ holds,
    $Y^{(1)}_j=Y^{(2)}_j$ for all $j\ne i$,
    so 
    \begin{align*}
    &Y_i^{(1)}-Y_i^{(2)}\mid Y^{(0)},\mathcal A^{(t)},\mathcal B^{(t)}
    =\varepsilon_1-\varepsilon_2\sim \mathcal{N}\left(0,\frac2{\Theta_{ii}}\right),
    \end{align*}
    and the dependence on $Y^{(0)}$ may be dropped.
\end{proof}

\subsection{Retrieving the estimator}
\label{ss:retrieving-diagonal}
We divide our sample into $M$ intervals of length $T$,
discarding intervals that do not satisfy $\mathcal B^{(t)}$,
sampling $Y_i^{(1)}-Y_i^{(2)}$ from the rest,
and treating intervals that fail to satisfy $\mathcal A^{(t)}$
as corruption
in order to predict $1/\Theta_{ii}$ with robust one-dimensional
variance estimation.

By \Cref{lem:strict,lem:pattern-independence}, the corruption rate among the
retained intervals is bounded by
\[
q_i:=\mathbb P[\neg\mathcal A^{(t)}\mid\mathcal B^{(t)}]
\le 1-e^{-dT}.
\]

\begin{lemma}\label{lem:chernoff1}
    In a Glauber dynamic of length $MT$,
    in at least $\frac1{16}MT^2$ intervals
    $\mathcal B^{(t)}$ holds
    with probability at least $1-\exp(-\frac1{64}MT^2)$.
\end{lemma}
\begin{proof}
By \Cref{lem:exhibit}, each interval satisfies $\mathcal B^{(t)}$ with
probability at least $T^2/8$. These events are independent across disjoint
intervals.
A Chernoff lower-tail bound therefore gives
\[
\mathbb P\left[
\sum_{r=1}^M \mathbf 1\{\mathcal B^{(r)}\}\le\frac12\frac{MT^2}{8}
\right]
\le
\exp\left(-\frac18\frac{MT^2}{8}\right),
\]
which is exactly the claimed bound.
\end{proof}

\begin{proof}[Proof of \Cref{lem:singlethetaii}]
    Choose
    \[
    T=\frac{\varepsilon}{5d}
    \qquad\text{and}\qquad
    MT^2=\frac{800\log(8/\delta)}{\varepsilon^2}.
    \]
    Then the actual corruption rate satisfies
    \[
    q_i
    \le 1-e^{-dT}
    \le 1-e^{-\varepsilon/5}
    \le \frac{\varepsilon}{5}.
    \]
    Since
    \[
    \exp\!\left(-\frac{MT^2}{64}\right)
    =
    \exp\!\left(-\frac{25\log(8/\delta)}{2\varepsilon^2}\right)
    \le \frac{\delta}{8},
    \]
    by \Cref{lem:chernoff1},
    with probability $1-\delta/2$ we see
    \[
    \frac{MT^2}{16}
    =
    \frac{50\log(8/\delta)}{\varepsilon^2}
    =
    \frac{2\log(8/\delta)}{(\varepsilon/5)^2}.
    \]

    Then, by \Cref{lem:robustvarest} with corruption parameter
    \(q_i\le\varepsilon/5\) and failure probability \(\delta/2\), we
    can estimate \(\sqrt{2/\Theta_{ii}}\) within additive error
    \[
    5\cdot\frac{\varepsilon}{5}\sqrt{\frac{2}{\Theta_{ii}}}
    =
    \varepsilon\sqrt{\frac{2}{\Theta_{ii}}},
    \]
    and hence within a factor \(1\pm\varepsilon\), with success probability
    \(1-\delta/2\).
    Taking a union bound gives the desired success rate.

    Finally, the required length of the trajectory is
    \[
    MT=\frac{4000d\log(8/\delta)}{\varepsilon^3}. \qedhere
    \]
\end{proof}

Finally,
we derive estimates for all $n$
diagonal entries $\Theta_{ii}$.
\begin{corollary}\label{coro:thetaii}
    Given a Glauber dynamic of length
    \[T=\frac{4000d\log(8n/\delta)}{\varepsilon^3},\] we may retrieve in polynomial
    time estimates $D_i$ for each $i$ such that, with probability at least
    $1-\delta$,
    \[\left\lvert
    \frac1{\sqrt{D_i}}-\frac1{\sqrt{\Theta_{ii}}}
    \right\rvert\le\frac\varepsilon{\sqrt{\Theta_{ii}}}\]
    for all $i=1,\ldots,n$.
\end{corollary}
\begin{proof}
Apply \Cref{lem:singlethetaii}
with error $\delta/n$.
Then, the probability of any error 
among the $n$ estimates is $\delta$
by union bound.
\end{proof}

\begin{algorithm}[H]
    \caption{$D=\text{EstimateDiagonal}(n,d,\alpha,\varepsilon,\delta)$}
    \label{alg:estdiag}
    \begin{algorithmic}[1]
        \STATE Set $T=\frac{\varepsilon}{5d}$ and $M=\frac{20000d^2\log(8n/\delta)}{\varepsilon^4}$.
        \STATE Observe Glauber trajectory of length $MT$,
        split into $M$ intervals $I_1$, $\ldots$, $I_M$ of length $T$.
        \FOR{$i=1,\ldots,n$}
            \STATE Let $S=\{1\le t\le M:\mathcal B^{(t)}\}$, where $\mathcal B^{(t)}$ is defined in \Cref{ss:propertystat3}.
            \IF{$|S|<\frac1{16}{MT^2}$}
                \STATE Return $\bot$.
            \ELSE
                \STATE By \Cref{lem:robustvarest}, estimate $\widehat\sigma^2$ from $\{Y^{(1)}_i-Y^{(2)}_i:t\in S\}$, where $Y_i^{(1)}$ and $Y_i^{(2)}$ are defined in \Cref{ss:propertystat3}.
                \STATE Set $D_i=2/\widehat\sigma^2$.
            \ENDIF
        \ENDFOR
        \STATE Return $D$.
    \end{algorithmic}
\end{algorithm}

\subsection{Normalizing the Gaussian}

\label{ss:normalizing}
Let $X^{(1)}$, $X^{(2)}$, $\ldots$
denote the updates from the Glauber dynamic.
As earlier,
define $X'^{(1)}$, $X'^{(2)}$, $\ldots$
so that \(X_i'^{(k)}:=\sqrt{\Theta_{ii}}\cdot X_i^{(k)}\) for all \(i,k\).
\begin{lemma}
\label{lem:thetaprime}
    The above $X'^{(1)}$, $X'^{(2)}$, $\ldots$
    is a Glauber trajectory 
    with precision matrix
    $\Theta'=\Theta_{\mathrm{diag}}^{-1/2}
    \Theta\Theta_{\mathrm{diag}}^{-1/2}$.
\end{lemma}
\begin{proof}
    Note that \(\Theta'_{ij}=\frac{\Theta_{ij}}
    {\sqrt{\Theta_{ii}\Theta_{jj}}}\).
    
    We may directly check
    that \Cref{def:cont-glauber-dynamics} is preserved.
    We have
    \begin{align*}
        x'_i&\mapsto
        \sqrt{\Theta_{ii}}
        \left(
        -\sum_{j=1}^n\frac{\Theta_{ij}}{\Theta_{ii}}
        \frac{x'_j}{\sqrt{\Theta_{jj}}}
        +\mathcal{N}\left(0,\frac1{\Theta_{ii}}\right)
        \right)\\
        &=-\sum_{j=1}^n\frac{\Theta_{ij}}
        {\sqrt{\Theta_{ii}\Theta_{jj}}}x_j'
        +\mathcal{N}(0,1).\qedhere
    \end{align*}
\end{proof}

To prove our main theorem,
we consider the trajectory
$\widehat X^{(1)}$, $\widehat X^{(2)}$, $\ldots$
defined by
\(\widehat X^{(k)}_i=\sqrt{D_i}\cdot X_i^{(k)}\) for all \(i,k\).
\paragraph{Putting the normalization together.}
\begin{proof}[Proof of \Cref{thm:normalization}]
    It suffices to show $\widehat X$
    is a Glauber dynamic with precision
    matrix $\widehat\Theta
    =D^{-1/2}\Theta D^{-1/2}$,
    i.e. \(\widehat\Theta_{ij}=\frac{\Theta_{ij}}{\sqrt{D_iD_j}}\).
    Again, we check \Cref{def:cont-glauber-dynamics}
    is preserved.
    We have
    \begin{align*}
    \widehat x_i
    &\mapsto\sqrt{D_i}
        \left(
        -\sum_{j=1}^n\frac{\Theta_{ij}}{\Theta_{ii}}
        \frac{\widehat x_j}{\sqrt{D_j}}
        +\mathcal{N}\left(0,\frac1{\Theta_{ii}}\right)
        \right)\\
    &=-\sum_{j=1}^n\frac{\widehat\Theta_{ij}}
    {\widehat\Theta_{ii}}\widehat x_j
    +\mathcal{N}\left(0,\frac1{\widehat\Theta_{ii}}\right).
    \end{align*}
    Finally, for independence reasons, we use only the first
    \[
    T_{\mathrm{diag}}
    :=
    \frac{4000d\log(8n/\delta)}{\varepsilon^3}
    \]
    updates of the trajectory to estimate $D$, discard that initial segment, and
    keep the remaining trajectory of length $T_{\mathrm{rest}}$ to produce the
    trajectory with the desired properties.
\end{proof}

In particular, note that by construction,
$\widehat\Theta$ is coordinate-wise
a $(1\pm O(\varepsilon))$-approximation for $\Theta'$
(defined in \Cref{lem:thetaprime}).

Concretely,
\begin{corollary}\label{coro:coordwise}
    There exist $c_1,\ldots,c_n\in1\pm\varepsilon$
    so that
    for each $i$ and $k$,
    we have $\widehat X_i^{(k)}=\frac1{c_i}X_i'^{(k)}$;
    moreover 
    $\widehat\Theta_{ij}=c_ic_j\Theta_{ij}'$
    for each $i$ and $j$.
\end{corollary}
\begin{proof}
    Indeed,
    \[\frac
    {X_i'^{(k)}}
    {\widehat X_i^{(k)}}
    =\sqrt{\frac{\Theta_{ii}}{D_i}}
    =:c_i
\in1\pm\varepsilon,\]
    and further
    \[\frac{\widehat\Theta_{ij}}{\Theta_{ij}'}
    =\frac{\sqrt{\Theta_{ii}\Theta_{jj}}}
    {\sqrt{D_iD_j}}
=c_ic_j.\qedhere\]
\end{proof}
Thus, the trajectory $\widehat X$ we constructed
has two properties:
\begin{itemize}
    \item it is itself a Glauber trajectory
        with diagonal entries $1\pm O(\varepsilon)$
        (and precision matrix $\widehat\Theta$);
    \item it is coordinate-wise a
        $(1\pm\varepsilon)$-approximation
        for a Glauber trajectory with diagonal entries 1
        (and precision matrix $\Theta'$).
\end{itemize}

\section{Main Algorithm}
\label{sec:main-algorithm}
We now move to structure-learning.
Our main result is as follows:
\begin{theorem}[\Cref{thm:main-theorem-structure-learning}]
\label{thm:structure2}
    Let $0<\alpha,\delta<1$ and $n,d\in\mathbb N$.
    Given a Glauber trajectory evolving according to an $(\alpha,d)$-sparse GGM
    with precision matrix $\Theta$ and whose length satisfies
    \[
        T\ge\frac{256\cdot257^5\, d^3\log(4n/\delta)}{\alpha^5},
    \]
    there is a polynomial-time algorithm
    that correctly outputs whether $i\sim j$ for each $i$ and $j$
    with probability $1-\delta$.
\end{theorem}

We also derive the following parameter-learning guarantees:
\begin{corollary}[\Cref{thm:main-theorem-parameter-learning}]
\label{coro:paramlearn}
    Let $0<\alpha,\delta,\varepsilon<1$,
    and let $n,d\in\mathbb N$.
    Given a Glauber trajectory evolving according to an $(\alpha,d)$-sparse GGM
    with precision matrix $\Theta$ and whose length satisfies
    \[
        T\ge\frac{256\cdot257^5\, d^3\log(4n/\delta)}{\alpha^5\varepsilon^5},
    \]
    there is a polynomial-time algorithm that outputs an estimate $\widehat\Theta$
    for $\Theta$ such that
    \(\left\lvert\widehat\Theta_{ij}-\Theta_{ij}\right\rvert
    \le\varepsilon\left\lvert\Theta_{ij}\right\rvert\)
    for each $i$ and $j$.
\end{corollary}

To this end,
we fix $i$ and $j$,
and evaluate the existence of each edge $i\sim j$
individually.
\begin{lemma}\label{lem:thetaij}
    For fixed $i$ and $j$,
    given a Glauber trajectory whose length satisfies
    \[
        T\ge\frac{128\cdot257^5\, d^3\log(16/\delta)}{\alpha^5},
    \]
    we may determine whether $i\sim j$ in polynomial time 
    with probability $1-\delta$.
\end{lemma}

\Cref{ss:propertystat} analyzes the statistic used to test whether $i\sim j$: first in the ideal normalized trajectory, then for the observable approximate trajectory, and finally after conditioning on a large denominator. \Cref{ss:retrieving-estimator} uses these ingredients to build the estimator and complete the proofs of \Cref{lem:thetaij}, \Cref{thm:structure2}, and \Cref{coro:paramlearn}.

A naive analysis based on the shorter ``$iji$'' pattern fails because the middle $j$-update depends on the earlier $i$-update, creating a dependence in the resulting ratio statistic. For this reason we instead study the ``$iiji$'' pattern, whose extra $i$-update breaks this dependence and leads to the statistic analyzed below. The observations remain dependent and can be corrupted by hidden neighbor updates, so after conditioning on $|\Delta_j|>c$ we estimate the resulting location parameter using \Cref{lem:advrobustmeanest}.

\subsection{Properties of the statistic}
\label{ss:propertystat}
Fix $i$ and $j$.
Let $T\le1/3$, let $I_t$ denote the time interval $[t,t+T)$,
and let \(P_{iiji}:=(i,i,j,i)\).
Let $\mathcal D^{(t)}$ denote the observable event that $I_t$ exhibits
$P_{iiji}$, and let $\mathcal C^{(t)}$ denote the event that $I_t$ is strict
with respect to $P_{iiji}$. By \Cref{lem:pattern-independence},
$\mathcal D^{(t)}$ and $\mathcal C^{(t)}$ are independent.

Normalize the Glauber trajectory
by \Cref{thm:normalization},
and let $\widehat X$ consist of
$(1\pm1/10)$ coordinate-wise approximations of $X'$.
We first analyze the normalized
Glauber dynamic $X'$.

\paragraph{The ideal normalized statistic.}

Conditioned on $\mathcal C^{(t)}$ and 
$\mathcal D^{(t)}$,
let $Y'^{(k)}$ denote the value of $X'$
at time $t+Tk/4$, for $k=0,1,2,3,4$.
Denote
\[
    \Delta'^{(t)}_j
    :=Y'^{(3)}_j-Y'^{(0)}_j
    \quad\text{and}\quad
    \Delta'^{(t)}_i
    :=Y'^{(4)}_i-Y'^{(1)}_i.
\]

\begin{lemma}\label{iijiStatisticFormula}
    We have
    \(\Delta'^{(t)}_i
    \mid \mathcal C^{(t)},\mathcal D^{(t)},Y'^{(0)},Y'^{(3)}
    \sim-\Theta_{ij}'\Delta'^{(t)}_j+\mathcal{N}(0,2)\).
\end{lemma}
\begin{proof}
    Throughout this proof we condition on $\mathcal C^{(t)}$ and $\mathcal D^{(t)}$.
    Let $\varepsilon_1,\varepsilon_2,\varepsilon_3,\varepsilon_4\sim\mathcal N(0,1)$
    denote the fresh noises coming from the last designated updates of $i$ in the
    first, second, and fourth quarters, and of $j$ in the third quarter,
    respectively.
    By the definition of $\mathcal D^{(t)}$ and $\mathcal C^{(t)}$, the interval
    exhibits $P_{iiji}$ and is strict with respect to it.

    First,
    \[
        Y_i'^{(1)}=-\sum_{k\ne i}\Theta_{ik}'Y_k'^{(0)}+\varepsilon_1,
    \]
    because no neighbor of $i$ changes during the first quarter. Likewise,
    \[
        Y_i'^{(2)}=-\sum_{k\ne i}\Theta_{ik}'Y_k'^{(0)}+\varepsilon_2,
    \]
    because the second quarter again contains an $i$-update and no neighbor of
    $i$ updates.

    Before the designated $j$-update in the third quarter, the only neighbor of
    $j$ that may have changed is $i$, so
    \[
        Y_j'^{(3)}
        =-\sum_{k\ne j}\Theta_{jk}'Y_k'^{(2)}+\varepsilon_3.
    \]
    Thus $Y'^{(3)}$ depends on $\varepsilon_2$ and $\varepsilon_3$, but not on
    $\varepsilon_1$.

    Finally, before the designated $i$-update in the fourth quarter, the only
    neighbor of $i$ that may have changed is $j$, so
    \[
        Y_i'^{(4)}
        =-\sum_{k\ne i,j}\Theta_{ik}'Y_k'^{(0)}
        -\Theta_{ij}'Y_j'^{(3)}+\varepsilon_4.
    \]
    Subtracting the expression for $Y_i'^{(1)}$ gives
    \[
        Y_i'^{(4)}-Y_i'^{(1)}
        =-\Theta_{ij}'\bigl(Y_j'^{(3)}-Y_j'^{(0)}\bigr)
        +\varepsilon_4-\varepsilon_1.
    \]
    Then $\varepsilon_4-\varepsilon_1\sim\mathcal N(0,2)$ is independent of
    $Y'^{(0)}$ and $Y'^{(3)}$. This proves the claim.
\end{proof}

\paragraph{Passing to the observable statistic.}
We now turn to the observable data from $\widehat X$.
Let $\widehat Y^{(k)}$ denote the value of $\widehat X$
at time $t+Tk/4$, for $k=0,1,2,3,4$.
Denote
\[
    \widehat\Delta_j^{(t)}
        =\widehat Y^{(3)}_j-\widehat Y^{(0)}_j
        \quad\text{and}\quad
        \widehat\Delta_i^{(t)}
        =\widehat Y^{(4)}_i-\widehat Y^{(1)}_i.
\]

\begin{lemma}\label{lem:hatsadness}
    For constants $c_i,c_j\in1\pm1/10$ as in Corollary~\ref{coro:coordwise},
    we have for all $t$ with
    $\mathcal C^{(t)}$ and $\mathcal D^{(t)}$ that
    \begin{align*}
        &\widehat\Delta_i^{(t)}
        \mid\mathcal C^{(t)},\mathcal D^{(t)},\widehat Y^{(0)},\widehat Y^{(3)}
        \sim-\frac{c_j}{c_i}\Theta_{ij}'\widehat\Delta_j^{(t)}
        +\mathcal{N}\left(0,\frac2{c_i^2}\right).
    \end{align*}
\end{lemma}
\begin{proof}
    By Corollary~\ref{coro:coordwise},
    we have $\widehat\Delta_i^{(t)}=\frac1{c_i}\Delta'^{(t)}_i$
    and $\widehat\Delta_j^{(t)}=\frac1{c_j}\Delta'^{(t)}_j$,
	    from which the claim follows.
\end{proof}

\paragraph{Conditioning on a large denominator.}
Let $\mathcal E^{(t)}$
denote the condition that 
\(\lvert\widehat\Delta_j^{(t)}\rvert\ge0.6\).
We will restrict our attention
to samples where $\mathcal E^{(t)}$ holds.

\begin{lemma}\label{lem:e}
We have \(\mathbb{P}[\mathcal E^{(t)} \mid \mathcal D^{(t)}, \widehat Y^{(0)}]>1/4\).
\end{lemma}
\begin{proof}
    As $c_j<1.1$,
    we have that \(\lvert\widehat\Delta_j^{(t)}\rvert\ge0.6\)
    is implied by \(\lvert\Delta'^{(t)}_j\rvert\ge0.66\),
    so it suffices to show \(\lvert\Delta'^{(t)}_j\rvert\ge0.66\)
    with probability more than $1/4$.

    Let $Z$ be the value of $X'$ right before
    the last $j$-update in the third quarter.
    Then, by \Cref{def:cont-glauber-dynamics},
    $Y'^{(3)}_j$ given $Z$ is a Gaussian with variance 1,
    i.e.
    \(\Delta'^{(t)}_j \mid Z\sim \mathcal{N}(m,1)\)
    for some $m$ dependent on $Z$.

    Therefore, for any fixed $Z$, we have
    \begin{align*}
    \mathbb{P}[\lvert\Delta_j'^{(t)}\rvert>0.66 \mid Z]
        &=\mathbb{P}_{x\sim \mathcal{N}(m,1)}[\lvert x\rvert\ge0.66 \mid Z]\\
        &\ge\mathbb{P}_{x\sim \mathcal{N}(\lvert m\rvert,1)}[x\ge0.66 \mid Z]\\
        &\ge\mathbb{P}_{x\sim \mathcal{N}(0,1)}[x\ge0.66]
        >1/4,
    \end{align*}
    implying that
    \[
        \mathbb{P}[\lvert\Delta_j'^{(t)}\rvert>0.66\mid \mathcal D^{(t)},\widehat Y^{(0)}]
        >\frac14,
    \]
    regardless of the realized value of $Z$.
\end{proof}
\begin{lemma}\label{lem:nomixing-ratio}
For some $\sigma<2.62$,
    we have 
    \[
        -\frac{\widehat\Delta^{(t)}_i}{\widehat\Delta^{(t)}_j}
        \;\Bigm|\;\widehat Y^{(0)},\widehat Y^{(3)},
        \mathcal C^{(t)},\mathcal D^{(t)},\mathcal E^{(t)}
        \sim\frac{c_j}{c_i}\Theta_{ij}'
        +\mathcal{N}(0,\sigma^2).
    \]
\end{lemma}
\begin{proof}
    Since $\mathcal E^{(t)}$
    is determined by $\widehat Y^{(0)}$ and $\widehat Y^{(3)}$,
    \Cref{lem:hatsadness} gives
    \begin{align*}
        &\frac{\widehat\Delta^{(t)}_i}{\widehat\Delta^{(t)}_j}
        \;\Bigm|\;\widehat Y^{(0)},\widehat Y^{(3)},
        \mathcal C^{(t)},\mathcal D^{(t)},\mathcal E^{(t)}
        \sim-\frac{c_j}{c_i}\Theta_{ij}'
        +\mathcal{N}\left(0,\frac2
            {c_i^2\left(\widehat\Delta_j^{(t)}\right)^2}
        \right),
    \end{align*}
    Since $c_i>0.9$ and \(\lvert\widehat\Delta_j^{(t)}\rvert>0.6\),
    the upper bound follows from
    \(\frac{\sqrt2}{c_i\lvert\widehat\Delta_j^{(t)}\rvert}<2.62\). \qedhere
\end{proof}

\subsection{Retrieving the estimator}
\label{ss:retrieving-estimator}
We divide our Glauber trajectory into $M$ intervals of length $T$,
for some $M$ and $T$ we will decide later
(so the trajectory has length $MT$).
We discard intervals that do not satisfy
$\mathcal D^{(t)}$ and $\mathcal E^{(t)}$,
sampling $-\widehat\Delta_i^{(t)}/\widehat\Delta_j^{(t)}$
from the rest,
and treating intervals that fail to satisfy $\mathcal C^{(t)}$
as corruption.

\begin{algorithm}[htbp]
    \caption{$\widetilde\Theta=\text{LearnGGM}(n,d,\alpha,\varepsilon,\delta)$}
    \label{alg:mainalgo}
    \begin{algorithmic}[1]
        \STATE Set $\phi=\alpha/257$, $T=\phi/d$, and $M=256d^4\log(4n/\delta)/\phi^6$.
        \STATE Let $\operatorname{diag}(\widetilde\Theta)=\text{EstimateDiagonal}(n,d,\alpha,1/10,\delta/2)$ from Algorithm~\ref{alg:estdiag}.
        \STATE Use the remaining trajectory.
        \STATE Let $E=\varnothing$.
        \STATE Observe Glauber trajectory of length $MT$, split into $M$ intervals $I_1$, $\ldots$, $I_M$ of length $T$.
        \FOR{$i=1,\ldots,n$}
            \FOR{$j=i+1,\ldots,n$}
                \STATE Let $S=\{1\le t\le M:\mathcal D^{(t)},\mathcal E^{(t)}\}$, where $\mathcal D^{(t)}$ and $\mathcal E^{(t)}$ are defined in \Cref{ss:propertystat}.
                \IF{$|S|<\frac1{4096}{MT^4}$}
                    \STATE Return $\bot$.
                \ELSE
                    \STATE By \Cref{lem:advrobustmeanest}, estimate the median $\widehat\nu$ from $\{-\widehat\Delta^{(t)}_i/(2.62\widehat\Delta^{(t)}_j):t\in S\}$, where $\widehat\Delta_i^{(t)}$ and $\widehat\Delta_j^{(t)}$ are defined in \Cref{ss:propertystat}, and set $\widehat\mu=-2.62\widehat\nu$.
                    \IF{$|\widehat\mu|>9\alpha/22$}
                        \STATE Set $E=E\sqcup\{(i,j)\}$.
                    \ENDIF
                \ENDIF
            \ENDFOR
        \ENDFOR
        \STATE Return $E$.
    \end{algorithmic}
\end{algorithm}

Let $\eta$ be an upper bound on the conditional corruption probability of an
accepted interval, i.e. on
\(\mathbb P[\neg\mathcal C^{(t)}\mid \mathcal D^{(t)},\mathcal E^{(t)},s]\)
for a realized start state $s$.
We first bound $\eta$.
\begin{lemma}
    We have
    \(\eta\le4(1-\exp(-2Td))\).
\end{lemma}
\begin{proof}
Let $s$ denote the start state of the interval. Since $P_{iiji}$ contains two
distinct indices, \Cref{lem:strict} gives
\[
\mathbb P[\neg\mathcal C^{(t)}]\le 1-e^{-2Td}.
\]
Also, by \Cref{lem:pattern-independence} and independence of the future clocks
from the start state,
\[
\mathbb P[\neg\mathcal C^{(t)}\mid \mathcal D^{(t)},s]\le 1-e^{-2Td}.
\]
Also \Cref{lem:e} gives
\[
\mathbb P[\mathcal E^{(t)}\mid \mathcal D^{(t)},s]>1/4.
\]
Therefore, for every start state $s$,
\begin{align*}
    \mathbb P[\neg\mathcal C^{(t)}\mid \mathcal D^{(t)},\mathcal E^{(t)},s]
    &=
    \frac{\mathbb P[\mathcal E^{(t)}\mid \neg\mathcal C^{(t)},\mathcal D^{(t)},s]\,
    \mathbb P[\neg\mathcal C^{(t)}\mid \mathcal D^{(t)},s]}
    {\mathbb P[\mathcal E^{(t)}\mid \mathcal D^{(t)},s]}
    \le 4(1-e^{-2Td}).
\end{align*}
Averaging over the start state yields the claim.
\end{proof}

We next show that we can collect
sufficiently many samples in which
$\mathcal D^{(t)}$ and $\mathcal E^{(t)}$ hold.
\begin{lemma}
\label{lem:chernoff2}
    For a Glauber dynamic of length $MT$
    (split into $M$ intervals of length $T$ as above),
    in at least $MT^4/4096$
    intervals do $\mathcal D^{(t)}$ and $\mathcal E^{(t)}$
    hold with probability at least
    $1-\exp(-MT^4/16384)$.
\end{lemma}
\begin{proof}
    Let
    \[
        A_t:=\mathbf 1\{\mathcal D^{(t)}\wedge \mathcal E^{(t)}\},
    \]
    and let $\mathcal F_t$ be the sigma-field generated by the first $t$
    intervals. By \Cref{lem:exhibit} applied to the pattern $P_{iiji}$,
    \[
        \mathbb P[\mathcal D^{(t)}]\ge \frac{T^4}{512}.
    \]
    Together with \Cref{lem:e}, this yields
    \[
        \mathbb E[A_t\mid \mathcal F_{t-1}]
        \ge \frac{T^4}{2048}
        =:p_0.
    \]
    Fix \(s>0\). Since \(0\le A_t\le 1\),
    \begin{align*}
        \mathbb E[e^{-sA_t}\mid \mathcal F_{t-1}]
        &=
        1-(1-e^{-s})\mathbb E[A_t\mid \mathcal F_{t-1}]
        \le
        e^{-(1-e^{-s})p_0}.
    \end{align*}
    Therefore
    \[
        Z_r
        :=
        \exp\!\left(
            -s\sum_{t=1}^r A_t+(1-e^{-s})p_0r
        \right)
    \]
    is a supermartingale. Using Markov's inequality with \(s=\log 2\),
    \begin{align*}
        \mathbb P\!\left[\sum_{t=1}^M A_t<\frac{Mp_0}{2}\right]
        &\le
        \exp\!\left(
            \frac{\log 2}{2}Mp_0-\frac12 Mp_0
        \right)
        <
        e^{-Mp_0/8}.
    \end{align*}
    Since \(Mp_0/2=MT^4/4096\) and \(Mp_0/8=MT^4/16384\), the claim follows.
\end{proof}

\paragraph{Putting the estimator together.}
\begin{proof}[Proof of \Cref{lem:thetaij}]
    First isolate the first
    \[
        T_{\mathrm{diag}}
        :=
        4{,}000{,}000d\log(32/\delta)
    \]
    of the trajectory. Applying \Cref{lem:singlethetaii} with
    \(\varepsilon=1/10\) and failure probability \(\delta/4\) to coordinates
    \(i\) and \(j\) recovers approximations \(D_{ii}\) and \(D_{jj}\) of the
    diagonal up to multiplicative factor \(1\pm 1/10\), and hence all necessary
    readings of \(\widehat X_i\) and \(\widehat X_j\), with probability at least
    \(1-\delta/2\). We then focus on the remainder of the trajectory.

    Set
    \[
        \phi:=\frac\alpha{257},
        \qquad
        T:=\frac\phi d,
        \qquad
        MT^4:=\frac{128\log(8/\delta)}{\phi^2}.
    \]
    The previous lemma gives the uniform conditional corruption bound
    \begin{align*}
        \mathbb P[\neg\mathcal C^{(t)}\mid \mathcal D^{(t)},\mathcal E^{(t)},s]
        &\le 4(1-e^{-2Td})
        =4(1-e^{-2\phi})
        <8\phi
    \end{align*}
    for every realized start state $s$.
    Set $\eta:=8\phi<1/10$.

    Since
    \[
        \exp\!\left(-\frac{MT^4}{16384}\right)
        =
        \exp\!\left(-\frac{\log(8/\delta)}{128\phi^2}\right)
        <\frac{\delta}{4},
    \]
    \Cref{lem:chernoff2} implies that with probability $1-\delta/4$ the number of
    accepted intervals is at least
    \[
        \frac{MT^4}{4096}
        =\frac{\log(8/\delta)}{32\phi^2}
        =\frac{2\log(8/\delta)}{\eta^2}.
    \]

    Condition on this sample-count event, and enumerate the accepted intervals
    in chronological order as $t_1<\cdots<t_N$. For $\ell=1,\dots,N$, let
    \[
        R_\ell:=-\frac{\widehat\Delta_i^{(t_\ell)}}{\widehat\Delta_j^{(t_\ell)}}
        \qquad\text{and}\qquad
        \xi^{(\ell)}:=\mathbf 1_{\neg\mathcal C^{(t_\ell)}}.
    \]
    Let $\mathcal F_\ell$ be the sigma-field generated by the diagonal-estimation
    phase together with the trajectory up to the end of interval $I_{t_\ell}$.
    By averaging the uniform bound above over the skipped intervals before the
    next accepted one, we obtain
    \[
        \mathbb E[\xi^{(\ell)}\mid \mathcal F_{\ell-1}]\le \eta
        \qquad\text{for all }\ell.
    \]
    Moreover, on the event $\xi^{(\ell)}=0$, \Cref{lem:nomixing-ratio} shows
    that, conditional on the full accepted interval, the clean sample
    $R_\ell/2.62$ has Gaussian tails dominated by
    \[
        \mathcal N\!\left(
            -\frac{1}{2.62}\frac{c_j}{c_i}\Theta_{ij}',1
        \right).
    \]
    Averaging over the accepted-interval randomness preserves these one-sided
    tail bounds conditional on $\mathcal F_{\ell-1}$.
    Therefore \Cref{lem:advrobustmeanest} applies to the rescaled samples
    $R_\ell/2.62$, and with success probability $1-\delta/4$ their sample median
    \(\widehat\nu\) satisfies
    \[
        \left|\widehat\nu+\frac{1}{2.62}\frac{c_j}{c_i}\Theta_{ij}'\right|
        <5\eta.
    \]
    Set \(\mu':=-2.62\widehat\nu\). Then
    \[
        \left|\mu'-\frac{c_j}{c_i}\Theta_{ij}'\right|
        <5\cdot 2.62\cdot\eta
        <
        \frac{9\alpha}{22}.
    \]

    However,
    \begin{itemize}
        \item If $i\nsim j$, then $\Theta_{ij}'=0$,
        so $|\mu'|<9\alpha/22$.
        \item If $i\sim j$, then $|\Theta_{ij}'|\ge \alpha$, so
        \[
            \left|\frac{c_j}{c_i}\Theta_{ij}'\right|
            \ge \frac{1-1/10}{1+1/10}\,\alpha
            >\frac{9\alpha}{11},
        \]
        and therefore $|\mu'|>9\alpha/22$.
    \end{itemize}
    Thus thresholding at $9\alpha/22$ determines whether $i\sim j$.
    A union bound over the diagonal-estimation phase, the sample-count event,
    and the robust-median step gives success probability at least $1-\delta$.

    Finally,
    \begin{align*}
        T_{\mathrm{diag}}+MT
        &\le
        4{,}000{,}000d\log(32/\delta)
        +\frac{128 d^3\log(8/\delta)}{\phi^5}
        <\frac{128\cdot257^5\, d^3\log(16/\delta)}{\alpha^5},
    \end{align*}
    where the diagonal-estimation prefix is absorbed by the second term since
    $d\ge1$ and $\alpha<1$. \qedhere
\end{proof}

\begin{proof}[Proof of \Cref{thm:structure2}]
It suffices to apply \Cref{lem:thetaij} with error $2\delta/n^2$
for each pair $(i,j)$.
Since
\[
    \log\!\left(\frac{16}{2\delta/n^2}\right)
    =\log\!\left(\frac{8n^2}{\delta}\right)
    \le 2\log\!\left(\frac{4n}{\delta}\right),
\]
the required trajectory length is bounded by the display in
\Cref{thm:structure2}. Then, the probability of any error among the
$\binom n2$ candidate edges is less than $\delta$ by union bound.
\end{proof}

\begin{proof}[Proof of Corollary~\ref{coro:paramlearn}]
    We follow the proof of \Cref{lem:thetaij}, but estimate each diagonal entry
    up to multiplicative factor \(1\pm\varepsilon/20\) and replace \(\alpha\) by
    \(\alpha\varepsilon\) in the off-diagonal estimator. The resulting trajectory
    length is exactly the displayed bound. Report the diagonal estimates as
    \(\widehat\Theta_{ii}\), and report \(\widehat\Theta_{ij}=0\) whenever the
    structure-learning step outputs \(i\nsim j\).

    For each pair \((i,j)\) with \(i\sim j\), the proof of \Cref{lem:thetaij}
    produces an estimate \(\widehat\mu_{ij}\) for
    \(\frac{c_j}{c_i}\Theta_{ij}'\) such that
    \[
        \left\lvert
        \widehat\mu_{ij}-\frac{c_j}{c_i}\Theta_{ij}'
        \right\rvert
        <
        \frac{9\alpha\varepsilon}{22}
        \le
        \frac{9}{22}\varepsilon\lvert\Theta_{ij}'\rvert,
    \]
    since \(\lvert\Theta_{ij}'\rvert\ge\alpha\) on edges. Also, because
    \(c_i,c_j\in1\pm\varepsilon/20\),
    \[
        \left\lvert\frac{c_j}{c_i}-1\right\rvert
        \le
        \frac{2(\varepsilon/20)}{1-\varepsilon/20}
        \le
        \frac{2}{19}\varepsilon.
    \]
    Therefore
    \[
        \left\lvert\widehat\mu_{ij}-\Theta_{ij}'\right\rvert
        \le
        \left(\frac{9}{22}+\frac{2}{19}\right)\varepsilon\lvert\Theta_{ij}'\rvert
        =
        \frac{215}{418}\varepsilon\lvert\Theta_{ij}'\rvert.
    \]

    Set
    \[
        \widehat\Theta_{ij}:=\widehat\mu_{ij}\sqrt{\widehat\Theta_{ii}\widehat\Theta_{jj}}.
    \]
    Since the diagonal estimates are within a factor \(1\pm\varepsilon/20\), we
    have
    \[
        \sqrt{\widehat\Theta_{ii}\widehat\Theta_{jj}}
        =
        (1\pm\varepsilon/20)\sqrt{\Theta_{ii}\Theta_{jj}}.
    \]
    Hence
    \begin{align*}
        \left\lvert\widehat\Theta_{ij}-\Theta_{ij}\right\rvert
        &\le
        \left(
            \frac{215}{418}\left(1+\frac1{20}\right)+\frac1{20}
        \right)
        \varepsilon\lvert\Theta_{ij}\rvert
        <
        0.591\,\varepsilon\lvert\Theta_{ij}\rvert
        <
        \varepsilon\lvert\Theta_{ij}\rvert.
    \end{align*}
    This proves the claim.
\end{proof}

\begin{remark}[On large constants]
    The constants in this section preclude practicality.
    The main source is the following tradeoff.
    To separate the cases $\Theta_{ij}'=0$ and $|\Theta_{ij}'|\ge\alpha$, the
    robust-median error after rescaling must be $O(\alpha)$. Since the hidden
    update corruption rate is $O(dT)$ for small window length $T$, one must take
    $T=O(\alpha/d)$. On the other hand, the visible pattern $(i,i,j,i)$ appears
    with probability $\Theta(T^4)$, so collecting the required
    $O(\alpha^{-2}\log(1/\delta))$ accepted windows forces a total trajectory
    length of order
    \[
        \Theta\!\left(\frac{d^3\log(1/\delta)}{\alpha^5}\right)
    \]
    up to constants. The concrete choice $T=\alpha/(257d)$ is simply a
    conservative instantiation of this tradeoff.
\end{remark}

\section{Learning With Mixing}
\label{sec:mixed_ub}

In this section, we give a more sample-efficient structure-learning algorithm
when the mixing time is known. Unlike the mixing-free algorithm of
\Cref{sec:main-algorithm}, we may now work with the shorter ``$iji$'' pattern.
The key point is that on a clean stationary $iji$ window, the $j$-increment and
the later $i$-increment form an exact bivariate Gaussian pair whose covariance
is $-\Theta'_{ij}$.

Our main result in this section is as follows.
\begin{theorem}\label{thm:mixlearn}
    Let $0<\alpha,\delta<1$ and $n,d\in\mathbb N$.
    Suppose we are given a Glauber trajectory evolving according to an
    $(\alpha,d)$-sparse GGM with precision matrix $\Theta$, and that
    $t_{\mathrm{mix}}(\varepsilon)$ is known. If the trajectory length satisfies
    \[
        T\ge
        \frac{80{,}000\log(2n/\delta)}{\alpha^2}
        \left(
            t_{\mathrm{mix}}\!\left(\frac{\alpha}{4000}\right)
            +\frac{8{,}000{,}000\,d^2}{\alpha^2}
        \right),
    \]
    then there is a polynomial-time algorithm that correctly outputs whether
    $i\sim j$ for each $i$ and $j$ with probability $1-\delta$.
\end{theorem}

Again, we fix $i$ and $j$, and evaluate the existence of the edge $i\sim j$
individually.
\begin{theorem}\label{thm:mixij}
    For fixed $i$ and $j$, given a Glauber trajectory whose length satisfies
    \[
        T\ge
        \frac{40{,}000\log(8/\delta)}{\alpha^2}
        \left(
            t_{\mathrm{mix}}\!\left(\frac{\alpha}{4000}\right)
            +\frac{8{,}000{,}000\,d^2}{\alpha^2}
        \right),
    \]
    we may determine whether $i\sim j$ with probability $1-\delta$.
\end{theorem}

\Cref{ss:spurproperties} analyzes the statistic in the normalized model and
then transfers it to the observable trajectory. \Cref{ss:edgedetection} uses
these ingredients together with an epoch decomposition and mixing to build the
edge test and complete the proofs of \Cref{thm:mixij,thm:mixlearn}.

\subsection{Properties of the statistic}
\label{ss:spurproperties}

Fix $i$ and $j$.
Let $T_0\le 1/3$, let $I_t$ denote the time interval $[t,t+T_0)$,
and let \(P_{iji}:=(i,j,i)\).
Let $\mathcal D^{(t)}$ denote the observable event that $I_t$ exhibits
$P_{iji}$, and let $\mathcal C^{(t)}$ denote the event that $I_t$ is strict
with respect to $P_{iji}$. By \Cref{lem:pattern-independence},
$\mathcal D^{(t)}$ and $\mathcal C^{(t)}$ are independent.

Normalize the Glauber trajectory by \Cref{thm:normalization}, and let
$\widehat X$ consist of $(1\pm\alpha/10)$ coordinate-wise approximations of
$X'$. We first analyze the normalized Glauber dynamic $X'$.

\paragraph{The ideal normalized statistic.}

Conditioned on $\mathcal C^{(t)}$ and $\mathcal D^{(t)}$, let $Y'^{(k)}$ denote
the value of $X'$ at time $t+kT_0/3$, for $k=0,1,2,3$. Denote
\[
\Delta'^{(t)}_j:=Y'^{(2)}_j-Y'^{(0)}_j,
\qquad
\Delta'^{(t)}_i:=Y'^{(3)}_i-Y'^{(1)}_i.
\]

\begin{lemma}\label{lem:iji-joint}
    Let $\theta:=\Theta'_{ij}$.
    If the interval begins from stationarity, so that
    $Y'^{(0)}\sim\mathcal{N}(0,\Sigma')$, then conditioned on
    $\mathcal C^{(t)}$ and $\mathcal D^{(t)}$ we have
    \[
        \begin{pmatrix}
            \Delta'^{(t)}_j\\[1mm]
            \Delta'^{(t)}_i
        \end{pmatrix}
        \sim
        \mathcal{N}\!\left(
            0,\;
            \begin{pmatrix}
                2 & -\theta\\
                -\theta & 2
            \end{pmatrix}
        \right).
    \]
\end{lemma}
\begin{proof}
    Let $R:=\Theta'Y'^{(0)}$.
    Since $Y'^{(0)}\sim\mathcal{N}(0,\Sigma')$ and
    $\Sigma'=(\Theta')^{-1}$, we have
        $R\sim\mathcal{N}(0,\Theta')$, with
    \begin{align*}
        \operatorname{Var}(R_i)=\operatorname{Var}(R_j)&=1\\
        \text{and}\quad\operatorname{Cov}(R_i,R_j)&=\theta.
    \end{align*}

    Let $\varepsilon_1,\varepsilon_2,\varepsilon_3\sim\mathcal{N}(0,1)$ be the
    fresh Gaussian noises corresponding to the last update of the designated
    coordinate in each of the three thirds. Because $I_t$ exhibits $P_{iji}$
    and is strict with respect to it, the endpoint of each prescribed coordinate
    is the fresh draw from the last update in that third. Hence
    \[
        Y'^{(1)}_i=Y'^{(0)}_i-R_i+\varepsilon_1.
    \]

    Before the middle $j$-update, the only neighbor of $j$ that may have
    changed is $i$, so
    \[
        \Delta'^{(t)}_j
        =
        Y'^{(2)}_j-Y'^{(0)}_j
        =
        -R_j+\theta R_i-\theta\varepsilon_1+\varepsilon_2.
    \]

    Similarly, before the final $i$-update, the only neighbor of $i$ that may
    have changed is $j$, so
    \[
        Y'^{(3)}_i
        =
        Y'^{(0)}_i-R_i-\theta\Delta'^{(t)}_j+\varepsilon_3,
    \]
    and therefore
    \[
        \Delta'^{(t)}_i
        =
        -\theta\Delta'^{(t)}_j+\varepsilon_3-\varepsilon_1.
    \]

    Everything is jointly Gaussian, so it remains only to compute the covariance
    matrix. First,
    \begin{align*}
        \operatorname{Var}(\Delta'^{(t)}_j)
        &=
        \operatorname{Var}(-R_j+\theta R_i)
        +\operatorname{Var}(-\theta\varepsilon_1+\varepsilon_2)\\
        &=
        (1-\theta^2)+(1+\theta^2)=2.
    \end{align*}
    Next,
    \begin{align*}
        \operatorname{Cov}(\Delta'^{(t)}_i,\Delta'^{(t)}_j)
        &=
        -\theta\operatorname{Var}(\Delta'^{(t)}_j)
        +\operatorname{Cov}(\varepsilon_3-\varepsilon_1,\Delta'^{(t)}_j)\\
        &=-2\theta+\theta=-\theta.
    \end{align*}
    Finally,
    \begin{align*}
        \operatorname{Var}(\Delta'^{(t)}_i)
        &=
        \theta^2\operatorname{Var}(\Delta'^{(t)}_j)
        +\operatorname{Var}(\varepsilon_3-\varepsilon_1)
        +2\operatorname{Cov}(-\theta\Delta'^{(t)}_j,\varepsilon_3-\varepsilon_1)\\
        &=-2\theta^2+2+2\theta^2
        =2.
    \end{align*}
    This proves the claim.
\end{proof}

\begin{corollary}\label{coro:iji-regression}
    On the same clean stationary interval,
    \[
        \Delta'^{(t)}_i
        =
        -\frac{\Theta'_{ij}}{2}\Delta'^{(t)}_j+\zeta^{(t)},
    \]
    where
    \[
        \zeta^{(t)}
        \sim
        \mathcal{N}\!\left(0,\,2-\frac{(\Theta'_{ij})^2}{2}\right)
    \]
    and $\zeta^{(t)}$ is independent of $\Delta'^{(t)}_j$.
\end{corollary}
\begin{proof}
    This is the Gaussian regression formula from \Cref{lem:iji-joint}, since
    \[
        \frac{\operatorname{Cov}(\Delta'^{(t)}_i,\Delta'^{(t)}_j)}
        {\operatorname{Var}(\Delta'^{(t)}_j)}
        =
        -\frac{\Theta'_{ij}}{2}.\qedhere
    \]
\end{proof}

\paragraph{Passing to the observable statistic.}

Let $\widehat Y^{(k)}$ denote the value of $\widehat X$ at time $t+kT_0/3$, for
$k=0,1,2,3$. Denote
\[
\widehat\Delta^{(t)}_j:=\widehat Y^{(2)}_j-\widehat Y^{(0)}_j,
\qquad
\widehat\Delta^{(t)}_i:=\widehat Y^{(3)}_i-\widehat Y^{(1)}_i.
\]

By \Cref{coro:coordwise}, there exist constants $c_i,c_j\in1\pm\alpha/10$ such
that
\[
    \widehat\Delta^{(t)}_j=\frac1{c_j}\Delta'^{(t)}_j
    \qquad\text{and}\qquad
    \widehat\Delta^{(t)}_i=\frac1{c_i}\Delta'^{(t)}_i.
\]
Hence \Cref{coro:iji-regression} yields
\[
    \widehat\Delta^{(t)}_i
    =
    -\frac{c_j}{2c_i}\Theta'_{ij}\,\widehat\Delta^{(t)}_j+\frac1{c_i}\zeta^{(t)},
\]
where $\zeta^{(t)}$ is independent of $\widehat\Delta^{(t)}_j$ and
\[
    \zeta^{(t)}
    \sim
    \mathcal{N}\!\left(0,\,2-\frac{(\Theta'_{ij})^2}{2}\right).
\]

\paragraph{Conditioning on a large denominator.}

Let $\mathcal E^{(t)}$ denote the event
\[
\mathcal E^{(t)}:=\{\lvert\widehat\Delta^{(t)}_j\rvert\ge 1.1\}.
\]

\begin{lemma}\label{lem:mixed-denominator}
    We have
    \[
        \mathbb{P}[\mathcal E^{(t)}\mid \mathcal D^{(t)},\widehat Y^{(0)}]
        >\frac29.
    \]
\end{lemma}
\begin{proof}
    Since $c_j<1.1$, the event $\lvert\widehat\Delta^{(t)}_j\rvert\ge1.1$ is
    implied by $\lvert\Delta'^{(t)}_j\rvert\ge1.21$.
    Let $Z$ be the value of $X'$ right before the last update of $j$ inside the
    middle third. Conditioned on $Z$, the endpoint $\Delta'^{(t)}_j$ is
    Gaussian with variance $1$, say
    $\Delta'^{(t)}_j\mid Z\sim\mathcal{N}(m,1)$ for some $m$.

    Let
    \[
        f(m):=\mathbb{P}_{x\sim\mathcal{N}(m,1)}[\lvert x\rvert\le1.21].
    \]
    Then $f$ is even, and for $m\ge0$,
    \[
        f'(m)=\varphi(1.21+m)-\varphi(1.21-m)\le0.
    \]
    Therefore $f$ is maximized at $m=0$, so
    \[
        \mathbb{P}[\lvert\Delta'^{(t)}_j\rvert\ge1.21\mid Z]
        \ge
        \mathbb{P}_{x\sim\mathcal{N}(0,1)}[\lvert x\rvert\ge1.21]
        >\frac29.
    \]
    Averaging over $Z$ proves the claim.
\end{proof}

\begin{lemma}\label{lem:mixed-ratio}
    Condition on a clean stationary interval and on $\mathcal E^{(t)}$.
    Then
    \[
        -2\frac{\widehat\Delta^{(t)}_i}{\widehat\Delta^{(t)}_j}
        =
        \frac{c_j}{c_i}\Theta'_{ij}+\sigma^{(t)}\varepsilon^{(t)},
    \]
    where $\varepsilon^{(t)}\sim\mathcal{N}(0,1)$,
    $\sigma^{(t)}$ is measurable with respect to
    $\widehat\Delta^{(t)}_j$, and
    \[
        \lvert\sigma^{(t)}\rvert<2.86.
    \]
\end{lemma}
\begin{proof}
    By the regression identity above,
    \[
        -2\frac{\widehat\Delta^{(t)}_i}{\widehat\Delta^{(t)}_j}
        =
        \frac{c_j}{c_i}\Theta'_{ij}
        -2\frac{\zeta^{(t)}}{c_i\widehat\Delta^{(t)}_j}.
    \]
    Since $\zeta^{(t)}$ is independent of $\widehat\Delta^{(t)}_j$ and has
    variance at most $2$, the second term equals
    $\sigma^{(t)}\varepsilon^{(t)}$ for some standard Gaussian
    $\varepsilon^{(t)}$ and
    \[
        \lvert\sigma^{(t)}\rvert
        \le
        \frac{2\sqrt2}{\lvert c_i\widehat\Delta^{(t)}_j\rvert}
        <
        \frac{2\sqrt2}{0.9\cdot1.1}
        <2.86,
    \]
    where we use $\mathcal E^{(t)}$.
\end{proof}

\subsection{Edge detection}
\label{ss:edgedetection}

We divide the trajectory into $M$ \emph{epochs}, each consisting of a mixing
period of length $t_{\mathrm{mix}}(\beta)$ followed by $N$ candidate intervals
of length $T_0$, where
\[
T_0=\frac{\alpha}{400d},
\quad
\beta=\frac{\alpha}{4000},
\quad
N=\frac{50}{T_0^3}.
\]

In the $k$th epoch, let $L_k$ denote the first candidate interval for which
$\mathcal D^{(t)}$ holds. If no such interval exists, the epoch is discarded.
If $L_k$ exists but $\mathcal E^{(t)}$ fails on that interval, the epoch is
also discarded. Otherwise we record the sample
\[
R_k:=-2\frac{\widehat\Delta_i^{(k)}}{\widehat\Delta_j^{(k)}},
\]
where $\widehat\Delta_i^{(k)}$ and $\widehat\Delta_j^{(k)}$ are computed from
the selected interval $L_k$.

The resulting procedure is given by \Cref{alg:mixedalgo}.
\begin{algorithm}[tbp]
    \caption{$E=\text{LearnGGMMixed}(n,d,\alpha,\delta,t_{\mathrm{mix}})$}
    \label{alg:mixedalgo}
    \begin{algorithmic}[1]
        \STATE Set $T_0=\alpha/(400d)$, $\beta=\alpha/4000$, $N=50/T_0^3$, and $M=80{,}000\log(2n/\delta)/\alpha^2$.
        \STATE Let $\operatorname{diag}(\widetilde\Theta)=\text{EstimateDiagonal}(n,d,\alpha,\alpha/10,\delta/2)$ from Algorithm~\ref{alg:estdiag}.
        \STATE Keep the unused part of the trajectory.
        \STATE Let $E=\varnothing$.
        \STATE Observe a trajectory of length $M(t_{\mathrm{mix}}(\beta)+NT_0)$, split into $M$ epochs.
        \FOR{$i=1,\ldots,n$}
            \FOR{$j=i+1,\ldots,n$}
                \STATE Let $S=\varnothing$.
                \FOR{$k=1,\ldots,M$}
                    \STATE After the mixing segment of epoch $k$, split the remainder into candidate intervals $I_{k,1},\ldots,I_{k,N}$ of length $T_0$.
                    \STATE Let $L_k$ be the first interval for which $\mathcal D^{(t)}$ holds.
                    \IF{$L_k$ exists and $\mathcal E^{(t)}$ holds on $L_k$}
                        \STATE Record $R_k=-2\widehat\Delta_i^{(k)}/\widehat\Delta_j^{(k)}$ and append it to $S$.
                    \ENDIF
                \ENDFOR
                \IF{$|S|<M/16$}
                    \STATE Return $\bot$.
                \ELSE
                    \STATE Let $\widehat\mu$ be the sample median of $S$.
                    \IF{$|\widehat\mu|>9\alpha/22$}
                        \STATE Set $E=E\sqcup\{(i,j)\}$.
                    \ENDIF
                \ENDIF
            \ENDFOR
        \ENDFOR
        \STATE Return $E$.
    \end{algorithmic}
\end{algorithm}

We now analyze the estimator.
\begin{lemma}\label{lem:epoch}
    Let $A_k$ denote the indicator that epoch $k$ records a sample. With the
    choice $N=50/T_0^3$, each epoch records a sample with conditional
    probability at least $1/8$ given the past. Consequently,
    \[
        \mathbb{P}\!\left[\sum_{k=1}^M A_k<\frac{M}{16}\right]
        \le e^{-M/128}.
    \]
\end{lemma}
\begin{proof}
    Fix an epoch and condition on the past.
    For a single candidate interval, \Cref{lem:exhibit} gives
    \[
        \mathbb{P}[\mathcal D^{(t)}]
        \ge
        \frac{T_0^3}{54}.
    \]
    Therefore the probability that some candidate interval satisfies
    $\mathcal D^{(t)}$ is at least
    \[
        1-\left(1-\frac{T_0^3}{54}\right)^N
        \ge
        1-e^{-50/54}
        >
        \frac35.
    \]

    Let $L_k$ be the first such interval.
    Conditioned on $L_k$ and on the past, \Cref{lem:mixed-denominator} gives
    \[
        \mathbb{P}[\mathcal E^{(t)}\text{ holds on }L_k\mid L_k,\text{ past}]
        >\frac29.
    \]
    Therefore the conditional probability that the epoch records a sample is
    at least
    \[
        \frac35\cdot\frac29=\frac{2}{15}>\frac18.
    \]

    Let $\mathcal F_k$ be the sigma-field generated by the first $k$ epochs, and
    define
    \[
        M_r:=\sum_{k=1}^r\left(A_k-\mathbb E[A_k\mid\mathcal F_{k-1}]\right).
    \]
    Then $(M_r,\mathcal F_r)$ is a martingale with increments in $[-1,1]$, and
    \(\sum_{k=1}^M\mathbb E[A_k\mid\mathcal F_{k-1}]\ge M/8\).
    Thus
    \[
        \mathbb{P}\!\left[\sum_{k=1}^M A_k<\frac{M}{16}\right]
        \le
        \mathbb{P}\!\left[M_M<-\frac{M}{16}\right]
        \le
        e^{-M/128}
    \]
    by Azuma--Hoeffding.
\end{proof}

\begin{lemma}\label{lem:mixed-corruption}
    Suppose
    \[
        T_0=\frac{\alpha}{400d}
        \qquad\text{and}\qquad
        \beta=\frac{\alpha}{4000}.
    \]
    Then, conditional on the past and on the event that an epoch records a
    sample, that sample is corrupted with probability less than $\alpha/35$.
\end{lemma}
\begin{proof}
    There are two sources of corruption.

    First, non-strict pattern intervals. Since future clocks are independent of
    the past and $P_{iji}$ contains two distinct indices,
    \Cref{lem:strict,lem:pattern-independence} give
    \[
        \mathbb{P}[\neg\mathcal C^{(t)}\mid \mathcal D^{(t)}]
        \le 1-e^{-2dT_0}.
    \]
    By \Cref{lem:mixed-denominator}, conditioning additionally on
    $\mathcal E^{(t)}$ inflates this by at most a factor of $9/2$, so the
    non-strictness contribution to the conditional corruption probability is at
    most
    \[
        \eta_{\mathrm{strict}}
        \le
        \frac{1-e^{-2dT_0}}{2/9}
        <
        9dT_0
        =
        \frac{9\alpha}{400}.
    \]

    Second, imperfect mixing.
    At the start of each epoch, after waiting $t_{\mathrm{mix}}(\beta)$, the
    epoch-start distribution is within TV-distance $\beta$ of stationarity. By
    \Cref{lem:mixingpersists}, the same is true at the start of any later
    candidate interval in the epoch. A maximal coupling with stationarity
    therefore fails with probability at most $\beta$.
    Since an epoch records a sample with conditional probability at least $1/8$
    by \Cref{lem:epoch}, conditioning on the epoch being recorded inflates this
    by at most a factor of $8$. Thus the mixing contribution to corruption is
    at most
    \[
        \eta_{\mathrm{mix}}\le 8\beta=\frac{\alpha}{500}.
    \]

    Combining the two bounds gives
    \[
        \eta
        \le
        \eta_{\mathrm{strict}}+\eta_{\mathrm{mix}}
        <
        \frac{9\alpha}{400}+\frac{\alpha}{500}
        =
        \frac{49\alpha}{2000}
        <
        \frac{\alpha}{35}.
        \qedhere
    \]
\end{proof}

\paragraph{Putting the edge test together.}
\begin{proof}[Proof of \Cref{thm:mixij}]
    First isolate the initial
    \[
        T_{\mathrm{diag}}
        :=
        \frac{4{,}000{,}000d\log(32/\delta)}{\alpha^3}
    \]
    portion of the trajectory, and apply \Cref{lem:singlethetaii} to retrieve
    approximations of the diagonal up to multiplicative factor $1\pm\alpha/10$,
    and hence all necessary readings of $\widehat X_i$ and $\widehat X_j$, with
    probability at least $1-\delta/2$.

    Now take
    \begin{align*}
        T_0=\frac{\alpha}{400d},\quad
        \beta&=\frac{\alpha}{4000},\quad
        N=\frac{50}{T_0^3},
       \quad M=\frac{40{,}000\log(8/\delta)}{\alpha^2}.
    \end{align*}
    By \Cref{lem:epoch}, with probability at least $1-\delta/4$, the number of
    recorded samples is at least
    \[
        \frac{M}{16}
        =
        \frac{2500\log(8/\delta)}{\alpha^2}.
    \]
    Let $k_1<\cdots<k_N$ denote the recorded epochs in chronological order, and
    condition on the sample-count event from \Cref{lem:epoch}, so that
    \[
        N\ge \frac{M}{16}=\frac{2500\log(8/\delta)}{\alpha^2}.
    \]
    For $\ell=1,\dots,N$, write
    \[
        R_\ell:=R_{k_\ell},
        \qquad
        \xi^{(\ell)}:=\mathbf 1_{\{\text{epoch }k_\ell\text{ is corrupted}\}}.
    \]
    Let $\mathcal F_\ell$ be the sigma-field generated by the trajectory up to
    the end of epoch $k_\ell$. By averaging \Cref{lem:mixed-corruption} over the
    skipped epochs before the next recorded one, we obtain
    \[
        \mathbb E[\xi^{(\ell)}\mid \mathcal F_{\ell-1}]\le \eta
        \qquad\text{for all }\ell,
    \]
    with
    \[
        \eta<\frac{\alpha}{35}.
    \]

    On the event $\xi^{(\ell)}=0$, \Cref{lem:mixed-ratio} implies that,
    conditional on $\mathcal F_{\ell-1}$, the clean sample $R_\ell/2.86$ has
    Gaussian tails dominated by
    \[
        \mathcal N\!\left(\frac{1}{2.86}\frac{c_j}{c_i}\Theta'_{ij},1\right).
    \]
    Let
    \[
        \mu:=\frac{c_j}{c_i}\Theta'_{ij}.
    \]
    Since
    \[
        \frac{M}{16}
        =
        \frac{2500\log(8/\delta)}{\alpha^2}
        >
        \frac{2\log(8/\delta)}{(\alpha/35)^2},
    \]
    \Cref{lem:advrobustmeanest} applies to the rescaled samples $R_\ell/2.86$.
    Therefore, with probability at least $1-\delta/4$, the sample median
    $\widehat\mu$ satisfies
    \[
        \lvert\widehat\mu-\mu\rvert
        <
        5\cdot2.86\cdot\frac{\alpha}{35}
        <
        \frac{9\alpha}{22}.
    \]

    Now:
    \begin{itemize}
        \item If $i\nsim j$, then $\Theta'_{ij}=0$, so $\mu=0$.
        \item If $i\sim j$, then $\lvert\Theta'_{ij}\rvert\ge\alpha$, and since
        $c_i,c_j\in1\pm\alpha/10$,
        \[
            \lvert\mu\rvert
            =
            \left\lvert\frac{c_j}{c_i}\Theta'_{ij}\right\rvert
            \ge
            \frac{1-\alpha/10}{1+\alpha/10}\alpha
            >
            \frac{9\alpha}{11}.
        \]
    \end{itemize}
    Hence thresholding at $9\alpha/22$ separates the two cases.

    A final union bound over the diagonal-estimation phase, the sample-count event, and the robust-median step gives success probability at least
    $1-\delta$, since each individual event has failure probability upper bounded by $\delta/4$.

    The per-epoch non-mixing cost is
    \[
        NT_0
        =
        \frac{50}{T_0^2}
        =
        50\left(\frac{400d}{\alpha}\right)^2
        =
        \frac{8{,}000{,}000\,d^2}{\alpha^2}.
    \]
    Therefore the total trajectory length is at most
    \begin{align*}
        &T_{\mathrm{diag}}+M\bigl(t_{\mathrm{mix}}(\beta)+NT_0\bigr)
        \le
        \frac{40{,}000\log(8/\delta)}{\alpha^2}
        \left(
            t_{\mathrm{mix}}\!\left(\frac{\alpha}{4000}\right)
            +\frac{8{,}000{,}000\,d^2}{\alpha^2}
        \right),
    \end{align*}
    where the diagonal-estimation prefix is absorbed by the
    $d^2/\alpha^4$ term since $d\ge1$ and $\alpha<1$.
\end{proof}

\begin{proof}[Proof of \Cref{thm:mixlearn}]
    Apply \Cref{thm:mixij} with error $2\delta/n^2$ for each pair $(i,j)$.
    Since
    \[
        \log\!\left(\frac{8}{2\delta/n^2}\right)
        =
        \log\!\left(\frac{4n^2}{\delta}\right)
        \le
        2\log\!\left(\frac{2n}{\delta}\right),
    \]
    the required trajectory length is bounded by the display in
    \Cref{thm:mixlearn}. A union bound over the $\binom n2$ candidate edges now
    gives success probability at least $1-\delta$.
\end{proof}

\section{Information-Theoretic Lower Bound}

\label{sec:info-theoretic-lower-bound}
In this section, we prove the lower bound by constructing a hard family of graphs, upper-bounding the pairwise KL divergence between the resulting trajectory distributions, and applying Fano's inequality. By \Cref{app:ct-dt-reduction}, this yields the corresponding continuous-time lower bound up to the usual factor of $n$.

Our strategy is to decompose the KL divergence step-wise. And upper bound the expected KL divergence introduced at each update, using the fact that KL divergence between Gaussians with the same variance is proportional to the square of the mean difference. Our initial decomposition is the same as that of \cite{tirukkonda2025structure}, but instead of upper bounding the KL divergence by the worst case update, our analysis involving the average case mean difference produces a $\log (n)$ improvement over the previous result. 

We formalize recovery through graph-learning tests:
\begin{definition}
    A \emph{graph-learning test} is a function that takes as input a Glauber trajectory $\mathcal T$ generated from an $(\alpha,d)$-GGM with precision matrix $\Theta$ and support graph $G$, together with the values $\alpha$ and $d$, and outputs an estimate $\widehat G$ for $G$. Its success event is $\widehat G=G$.
\end{definition}

From now on, we consider only well-mixed Glauber trajectories:
\begin{definition}
    Let $\mathcal T(N,\Theta)$ denote a Glauber trajectory,
    with $N$ updates, with precision matrix $\Theta$,
    and with starting point sampled from $\mathcal{N}(0,\Theta^{-1})$.
\end{definition}

Our main result in this section is as follows:
\begin{theorem}\label{thm:infolb}
    Let $n\ge16$ and $N$ be positive integers, and let $0<\alpha<1/4$.
    There exists a set $S$ of $2n$-dimensional $(\alpha,1)$-sparse GGMs such that,
    if $\Theta$ is sampled uniformly from $S$
    and the trajectory $\mathcal T(N,\Theta)$ is generated therefrom,
    then the success rate of a graph-learning test
    is at most $1/2$ whenever \(N\le\frac{n\log n}{8\alpha^2}\).
\end{theorem}
Note that $(\alpha,1)$-sparse GGMs are also $(\alpha,d)$-sparse
for every positive integer $d$,
so this bound holds for all $d$.

\subsection{The class of graphs}
Fix $n$ a positive integer,
and consider the graph $G_0$ of sparsity 1
on vertices labeled 1, 2, $\ldots$, $2n$,
with an edge $(2k-1)\sim2k$ for each $k=1,\ldots,n$,
and no other edges.
Further, let $G_k$ be a copy of $G_0$,
but with the edge $(2k-1)\sim2k$ removed,
for $k=1,\ldots,n$.

For each graph $G_k$,
define $\Theta_k$ as the following $(\alpha,1)$-sparse GGM
with underlying graph $G_k$:
\[(\Theta_k)_{ij}=\begin{cases}
    1&i=j\\
    \alpha&i\sim j\text{ in $G_k$}\\
    0&\text{else.}
\end{cases}.\]
Our set of GGMs is $S=\{\Theta_1,\ldots,\Theta_n\}$.

Notice that each GGM
is the direct sum of $2\times2$ matrices,
among which at least one is the identity matrix $I_2$,
and the remaining are
\[M_\alpha=\begin{bmatrix}
    1&\alpha\\\alpha&1
\end{bmatrix}.\]
The eigenvalues of $M_\alpha$ are $1-\alpha>0$ and $1+\alpha>0$,
so all $\Theta_k$ are positive definite.

\subsection{A bound on KL-divergence}
We now upper-bound the pairwise KL divergence between the candidate trajectory distributions. For a state $X$ before the $i$th update, write \(\mathcal T(i,\Theta_r)_{X^{(i)}\mid X^{(i-1)}=X}\) for the conditional distribution of the post-update state under $\Theta_r$, and \(\mathcal T(i,\Theta_r)_{X^{(i)}_\ell\mid X^{(i-1)}=X}\) for its $\ell$th coordinate, where $r\in\{0,1\}$.
\begin{lemma}\label{lem:singlekl}
    For each $k=1,\ldots,n$, we have
    \[\KL(\mathcal T(N,\Theta_k)\parallel\mathcal T(N,\Theta_0))
    \le\left(\frac Nn+1\right)\alpha^2.\]
\end{lemma}

\begin{proof}[Proof of \Cref{lem:singlekl}]
    Without loss of generality $k=1$.
    Recall both trajectories $\mathcal T(N,\Theta_0)$ and $\mathcal T(N,\Theta_k)$
    begin mixed, so we first compute the KL-divergence
    of their initial positions, by
    \begin{align*}
    \KL(\mathcal T(0,\Theta_1)_{X^{(0)}}\parallel\mathcal T(0,\Theta_0)_{X^{(0)}})
        &=\KL(\mathcal{N}(0,\Theta_1^{-1})\parallel\mathcal{N}(0,\Theta_0^{-1}))\\
        &=\KL(\mathcal{N}(0,M_\alpha)\parallel\mathcal{N}(0,I_2))\\
        &=-\frac12\ln(1-\alpha^2)<\alpha^2, 
    \end{align*}
    since $\alpha<1/4$.

    Now, by coupling, we may assume the two trajectories share
    the same sequence of updates.
    Suppose both trajectories have reached position $X$ before the $i$th update,
    and on the $i$th update, both trajectories update the $\ell$th coordinate.
    We explicitly write down the distribution of the $\ell$th coordinate
    after this Glauber update:
    \begin{align*}
        \mathcal T(i,\Theta_0)_{X^{(i)}_\ell\mid X^{(i-1)}=X}
        =-\sum_{j\sim \ell}-(\Theta_0)_{j\ell}X_j+\mathcal{N}(0,1),\\
\mathcal T(i,\Theta_1)_{X^{(i)}_\ell\mid X^{(i-1)}=X}
=-\sum_{j\sim \ell}-(\Theta_1)_{j\ell}X_j+\mathcal{N}(0,1)
\end{align*}
But recall $(\Theta_0)_{j\ell}=(\Theta_1)_{j\ell}$ unless
$\{j,\ell\}=\{1,2\}$, in which case
$(\Theta_0)_{12}-(\Theta_1)_{12}=\alpha$.
Hence, the two distributions are normal with the same variance,
and the means differ by $\alpha X_2$ if $\ell=1$,
by $\alpha X_1$ if $\ell=2$, and by 0 otherwise.

It follows that the KL-divergence for the $i$th update is given by
    \begin{align*}
    \KL\Big(\mathcal T(i,\Theta_1)_{X^{(i)}\mid X^{(i-1)}=X}\parallel\mathcal T(i,\Theta_0)_{X^{(i)}\mid X^{(i-1)}=X}\Big)
        &=\KL\Big(\mathcal T(i,\Theta_1)_{X^{(i)}_\ell\mid X^{(i-1)}=X}
        \parallel
\mathcal T(i,\Theta_0)_{X^{(i)}_\ell\mid X^{(i-1)}=X}\Big)\\
&=\mathop{\mathbb E}_\ell\left[\begin{cases}
    \frac12\alpha^2X_2^2&\ell=1\\
    \frac12\alpha^2X_1^2&\ell=2\\
    0&\ell\ge3
\end{cases}\right]
=\frac{\alpha^2}{4n}\left(X_1^2+X_2^2\right).
    \end{align*}

    But by \Cref{lem:stationary},
    we always have $X\sim \mathcal{N}(0,\Theta_1^{-1})$,
    so
    \(\mathop{\mathbb E}_X\left[X_1^2\right]=(\Theta_1^{-1})_{11}
    =\frac1{1-\alpha^2}<2\),
    and thus
    \begin{align*}
    &\mathop{\mathbb E}_{X^{(i-1)}}\Big[
    \KL\Big(\mathcal T(i,\Theta_1)_{X^{(i)}\mid X^{(i-1)}=X}
    \parallel\mathcal T(i,\Theta_0)_{X^{(i)}\mid X^{(i-1)}=X}\Big)
    \Big]
    =\frac{\alpha^2}{2n(1-\alpha^2)}
    <\frac{\alpha^2}n.
    \end{align*}
We conclude by linearity of expectation,
\begin{align*}
\KL(\mathcal T(N,\Theta_1)\parallel\mathcal T(N,\Theta_0))
&=\KL\Big(\mathcal T(0,\Theta_1)_{X^{(0)}}\parallel\mathcal T(0,\Theta_0)_{X^{(0)}}\Big)\\
&\qquad+\sum_{i=1}^N\mathop{\mathbb E}_{X^{(i-1)}}\Big[
    \KL\Big(\mathcal T(i,\Theta_1)_{X^{(i)}\mid X^{(i-1)}=X}
    \parallel\mathcal T(i,\Theta_0)_{X^{(i)}\mid X^{(i-1)}=X}\Big)
\Big]\\
&<\alpha^2+\frac{N\alpha^2}n.\qedhere
\end{align*}
\end{proof}

\subsection{Fano's method}
We finish by using Fano's method to bound the success rate of 
a graph-learning test.
\begin{lemma}\label{lem:fromfano}
    The success rate of a graph-learning test is bounded by
    \[\mathbb{P}\left[\widehat G=G\right]
    \le\frac{(N+n)\alpha^2}{n\log n}+\frac{\log2}{\log n}.\]
\end{lemma}
\begin{proof}
    Let $V$ be sampled uniformly from $\{1,\ldots,n\}$.
    Let $\mathcal T$ denote the observed trajectory.
    We may then bound the mutual information between $V$ and the observed trajectory by
    \begin{align*}
I(V;\mathcal T)
    &\le\frac1n\sum_{i=1}^n\KL(\mathcal T(N,\Theta_i)\parallel\mathcal T(N,\Theta_0))
    <\alpha^2+\frac{N\alpha^2}n.
\end{align*}
By Fano's inequality, we have
\begin{align*}
\mathbb{P}\left[\widehat G=G\right]
&\le\frac{I(V;\mathcal T)+\log 2}{\log n}
\le\frac{(N+n)\alpha^2}{n\log n}+\frac{\log2}{\log n}.\qedhere
\end{align*}
\end{proof}
\begin{proof}[Proof of \Cref{thm:infolb}]
With our given assumptions, we have
\(N+n\le\frac{n\log n}{4\alpha^2}\),
and so by \Cref{lem:fromfano}, we have
\begin{align*}
\mathbb{P}\left[\widehat G=G\right]
&\le\frac{(N+n)\alpha^2}{n\log n}+\frac{\log2}{\log n}
\le\frac14+\frac14=\frac12.\qedhere
\end{align*}
\end{proof}

\section*{Acknowledgements}
The authors thank Jason Gaitonde for helpful discussions and comments on an earlier version of this manuscript. They also thank Jonathan Bloom and Roman Bezrukavnikov for organizing the SPUR program at MIT, where this project began.

\bibliographystyle{alpha}
\bibliography{bib}

@article{gaitonde2025better,
  title={Better Models and Algorithms for Learning Ising Models from Dynamics},
  author={Gaitonde, Jason and Moitra, Ankur and Mossel, Elchanan},
  journal={arXiv preprint arXiv:2507.15173},
  year={2025}
}

@inproceedings{blanc2026nasty,
  title={Is nasty noise actually harder than malicious noise?},
  author={Blanc, Guy and Huang, Yizhi and Malkin, Tal and Servedio, Rocco A},
  booktitle={Proceedings of the 2026 Annual ACM-SIAM Symposium on Discrete Algorithms (SODA)},
  pages={6767--6787},
  year={2026},
  organization={SIAM}
}

@article{bhushan2019using,
  title={Using a Gaussian graphical model to explore relationships between items and variables in environmental psychology research},
  author={Bhushan, Nitin and Mohnert, Florian and Sloot, Daniel and Jans, Lise and Albers, Casper and Steg, Linda},
  journal={Frontiers in psychology},
  volume={10},
  pages={1050},
  year={2019},
  publisher={Frontiers Media SA}
}

@article{cerchiello2016conditional,
  title={Conditional graphical models for systemic risk estimation},
  author={Cerchiello, Paola and Giudici, Paolo},
  journal={Expert systems with applications},
  volume={43},
  pages={165--174},
  year={2016},
  publisher={Elsevier}
}

@inproceedings{gaitonde2024unified,
  title={A unified approach to learning Ising models: Beyond independence and bounded width},
  author={Gaitonde, Jason and Mossel, Elchanan},
  booktitle={Proceedings of the 56th Annual ACM Symposium on Theory of Computing},
  pages={503--514},
  year={2024}
}

@article{roberts1997updating,
  title={Updating schemes, correlation structure, blocking and parameterization for the Gibbs sampler},
  author={Roberts, Gareth O and Sahu, Sujit K},
  journal={Journal of the Royal Statistical Society Series B: Statistical Methodology},
  volume={59},
  number={2},
  pages={291--317},
  year={1997},
  publisher={Oxford University Press}
}

@article{amit1991rates,
  title={On rates of convergence of stochastic relaxation for Gaussian and non-Gaussian distributions},
  author={Amit, Yali},
  journal={Journal of Multivariate Analysis},
  volume={38},
  number={1},
  pages={82--99},
  year={1991},
  publisher={Elsevier}
}

@article{natarajan1995sparse,
  title={Sparse approximate solutions to linear systems},
  author={Natarajan, Balas Kausik},
  journal={SIAM journal on computing},
  volume={24},
  number={2},
  pages={227--234},
  year={1995},
  publisher={SIAM}
}

@inproceedings{zhang2014lower,
  title={Lower bounds on the performance of polynomial-time algorithms for sparse linear regression},
  author={Zhang, Yuchen and Wainwright, Martin J and Jordan, Michael I},
  booktitle={Conference on Learning Theory},
  pages={921--948},
  year={2014},
  organization={PMLR}
}

@article{zerenner2014gaussian,
  title={A Gaussian graphical model approach to climate networks},
  author={Zerenner, Tanja and Friederichs, Petra and Lehnertz, Klaus and Hense, Andreas},
  journal={Chaos: An Interdisciplinary Journal of Nonlinear Science},
  volume={24},
  number={2},
  year={2014},
  publisher={AIP Publishing}
}

@article{huang2010learning,
  title={Learning brain connectivity of Alzheimer's disease by sparse inverse covariance estimation},
  author={Huang, Shuai and Li, Jing and Sun, Liang and Ye, Jieping and Fleisher, Adam and Wu, Teresa and Chen, Kewei and Reiman, Eric and Alzheimer's Disease NeuroImaging Initiative and others},
  journal={NeuroImage},
  volume={50},
  number={3},
  pages={935--949},
  year={2010},
  publisher={Elsevier}
}

@article{dyrba2020gaussian,
  title={Gaussian graphical models reveal inter-modal and inter-regional conditional dependencies of brain alterations in alzheimer's disease},
  author={Dyrba, Martin and Mohammadi, Reza and Grothe, Michel J and Kirste, Thomas and Teipel, Stefan J},
  journal={Frontiers in aging neuroscience},
  volume={12},
  pages={99},
  year={2020},
  publisher={Frontiers Media SA}
}

@article{krumsiek2011gaussian,
  title={Gaussian graphical modeling reconstructs pathway reactions from high-throughput metabolomics data},
  author={Krumsiek, Jan and Suhre, Karsten and Illig, Thomas and Adamski, Jerzy and Theis, Fabian J},
  journal={BMC systems biology},
  volume={5},
  number={1},
  pages={21},
  year={2011},
  publisher={Springer}
}

@article{yi2022information,
  title={Information-incorporated Gaussian graphical model for gene expression data},
  author={Yi, Huangdi and Zhang, Qingzhao and Lin, Cunjie and Ma, Shuangge},
  journal={Biometrics},
  volume={78},
  number={2},
  pages={512--523},
  year={2022},
  publisher={Oxford University Press}
}

@inproceedings{gaitonde2025bypassing,
  title={Bypassing the Noisy Parity Barrier: Learning Higher-Order Markov Random Fields from Dynamics},
  author={Gaitonde, Jason and Moitra, Ankur and Mossel, Elchanan},
  booktitle={Proceedings of the 57th Annual ACM Symposium on Theory of Computing},
  pages={348--359},
  year={2025}
}

@inproceedings{misra2020information,
  title={Information theoretic optimal learning of gaussian graphical models},
  author={Misra, Sidhant and Vuffray, Marc and Lokhov, Andrey Y},
  booktitle={Conference on Learning Theory},
  pages={2888--2909},
  year={2020},
  organization={PMLR}
}

@article{bresler2017learning,
  title={Learning graphical models from the Glauber dynamics},
  author={Bresler, Guy and Gamarnik, David and Shah, Devavrat},
  journal={IEEE Transactions on Information Theory},
  volume={64},
  number={6},
  pages={4072--4080},
  year={2017},
  publisher={IEEE}
}

@inproceedings{wang2010information,
  title={Information-theoretic bounds on model selection for Gaussian Markov random fields},
  author={Wang, Wei and Wainwright, Martin J and Ramchandran, Kannan},
  booktitle={2010 IEEE International Symposium on Information Theory},
  pages={1373--1377},
  year={2010},
  organization={IEEE}
}

@inproceedings{tirukkonda2025structure,
  title={Structure learning in Gaussian graphical models from Glauber dynamics},
  author={Tirukkonda, Vignesh and Rayas, Anirudh and Dasarathy, Gautam},
  booktitle={2025 IEEE International Symposium on Information Theory (ISIT)},
  pages={1--6},
  year={2025},
  organization={IEEE}
}
\appendix
\section{Robust Estimators}
\label{sec:robust-estimates}
In this appendix, we prove the two robust-estimation lemmas stated in
\Cref{sec:preliminaries}, namely \Cref{lem:robustvarest,lem:advrobustmeanest}.
The variance proof is based
on controlling the relevant quantile under contamination, while the mean proof
uses a martingale argument together with Azuma--Hoeffding.
\robustvarianceestimator*
\begin{proof}
Let $\Phi$ and $\varphi$ denote the cdf and pdf, respectively, of the standard
Gaussian, and set
$q:=\Phi^{-1}(3/4)$.
Let $M$ be the sample median of the absolute values of the observed samples, and
set
$\widehat\sigma:=\frac{M}{q}$.
By scaling, it suffices to consider the case $\sigma=1$. Then the absolute value
of an uncorrupted sample has cdf
$G(t):=2\Phi(t)-1$.
Let $G_n$ denote the empirical cdf of the clean absolute values, let $H_n$
denote the empirical cdf of the observed absolute values, and let
$k:=\sum_{i=1}^n \xi_i$
be the number of corruptions.

By Dvoretzky--Kiefer--Wolfowitz with tolerance $\eta/2$,
\[
\mathbb P\!\left[\sup_{x\in\mathbb R}\lvert G_n(x)-G(x)\rvert>\frac{\eta}{2}\right]
\le 2e^{-n\eta^2/2}\le \frac{\delta}{2}.
\]
Also, a multiplicative Chernoff bound with parameter $1/3$ gives
\begin{align*}
\mathbb P\!\left[k>\frac{4}{3}\eta n\right]
&\le
\exp\!\left(
-\left(\frac43\log\frac43-\frac13\right)\eta n
\right)
\le e^{-\eta n/20}\le \frac{\delta}{4},
\end{align*}
where the last inequality uses
\[
n\ge \frac{2\log(4/\delta)}{\eta^2}
\qquad\text{and}\qquad
\eta\le \frac{1}{10}.
\]
Therefore, with probability at least $1-3\delta/4$,
\[
\sup_{x\in\mathbb R}\big\lvert G_n(x)-G(x)\big\rvert\le \frac{\eta}{2}
\qquad\text{and}\qquad
k\le \frac{4}{3}\eta n.
\]

Work on this event. Since $H_n$ is obtained from $G_n$ by changing at most $k$
sample points,
\begin{align*}
\sup_{x\in\mathbb R}\big\lvert H_n(x)-G(x)\big\rvert
&\le
\sup_{x\in\mathbb R}\big\lvert H_n(x)-G_n(x)\big\rvert
+
\sup_{x\in\mathbb R}\big\lvert G_n(x)-G(x)\big\rvert\\
&\le \frac{k}{n}+\frac{\eta}{2}
\le \frac{11\eta}{6}.
\end{align*}

Since $M$ is a median of the observed absolute values, we have
$H_n(M^-)\le 1/2\le H_n(M)$. Because $G$ is continuous and increasing, it
follows that
\[
\frac12-\frac{11\eta}{6}\le G(M)\le \frac12+\frac{11\eta}{6}.
\]
Equivalently,
\[
\frac34-\frac{11\eta}{12}\le\Phi(M)\le\frac34+\frac{11\eta}{12}.
\]
Since $\eta\le 1/10$, this places $M$ in the interval
\[
\Phi^{-1}\left(\frac{79}{120}\right)\le M\le\Phi^{-1}\left(\frac{101}{120}\right).
\]

Now $\Phi^{-1}$ is increasing and convex on $(1/2,1)$, so its deviation from
$q=\Phi^{-1}(3/4)$ is bounded linearly on this interval. Because
\[
\frac{11\eta}{12}\le \frac{11}{120}\cdot 10\eta,
\]
we obtain
\begin{align*}
\lvert M-q\rvert
\le
10\eta\cdot
\max\!\bigg\{
\Phi^{-1}\left(\frac{101}{120}\right)-q,\;
q-\Phi^{-1}\left(\frac{79}{120}\right)
\bigg\}.
\end{align*}
Finally,
\begin{align*}
\lvert\widehat\sigma-1\rvert
&=
\frac{\lvert M-q\rvert}{q}\\
&\le
\frac{10\eta}{\Phi^{-1}(3/4)}
\max\!\bigg\{
\Phi^{-1}\left(\frac{101}{120}\right)-\Phi^{-1}\left(\frac34\right),\;
\Phi^{-1}\left(\frac34\right)-\Phi^{-1}\left(\frac{79}{120}\right)
\bigg\}\\
&<5\eta.
\end{align*}
This proves the lemma.
\end{proof}

\robustmeanestimator*

\begin{proof}
Let $(\mathcal F_\ell)_{\ell=0}^n$ be the filtration from the statement, and
let $\widehat\mu$ denote the sample median. Let $\Phi$ and $\varphi$ denote the
cdf and pdf, respectively, of a standard Gaussian. Let \(\mathds 1_{\ell}\)
indicate the event that either \(\xi^{(\ell)}=1\) or
\(x^{(\ell)}>\mu+5\eta\). Then
\begin{align*}
    \mathbb E[\mathds 1_{\ell}\mid\mathcal F_{\ell-1}]
    &\le
    \mathbb E[\xi^{(\ell)}\mid \mathcal F_{\ell-1}] 
    +
    \mathbb P\!\left[
        x^{(\ell)}>\mu+5\eta,\ \xi^{(\ell)}=0
        \,\middle|\, \mathcal F_{\ell-1}
    \right] \\
    &\le
    \eta+
    \mathbb P\!\left[
        x^{(\ell)}>\mu+5\eta
        \,\middle|\, \mathcal F_{\ell-1},
        \xi^{(\ell)}=0
    \right] \\
    &\le 1+\eta-\Phi(5\eta).
\end{align*}
Therefore,
\[
M_k=\sum_{\ell=1}^k\left(
\mathds 1_{\ell}
-\mathbb E[\mathds 1_{\ell}\mid\mathcal F_{\ell-1}]
\right)
\]
is a martingale with bounded increments in \([-1,1]\).

By Azuma--Hoeffding (\Cref{fact:azuma-hoeffding}),
\begin{align*}
    \mathbb{P}[\widehat\mu\ge\mu+5\eta]
    &\le\mathbb{P}\left[\sum_{\ell=1}^n\mathds 1_{\ell}\ge\frac n2\right]\\
    &\le\mathbb{P}\left[M_n\ge n\left(\Phi(5\eta)-\eta-\frac12\right)\right]\\
    &\le\exp\left(-2n\left(\Phi(5\eta)-\eta-\frac12\right)^2\right).
\end{align*}
For \(\eta\le1/10\), we have \(\varphi(5\eta)>3/10\), and hence
\begin{align*}
    \Phi(5\eta)-\eta-\frac12
    &=-\eta+\int_0^{5\eta}\varphi(x)\,dx
    >-\eta+5\eta\varphi(5\eta)>\eta/2.
\end{align*}
Substituting this bound yields
\[
\mathbb{P}[\widehat\mu\ge\mu+5\eta]\le\exp(-n\eta^2/2).
\]

The lower tail is identical: if \(\widetilde{\mathds 1}_{\ell}\) denotes the
event that either \(\xi^{(\ell)}=1\) or \(x^{(\ell)}<\mu-5\eta\), the same
argument gives
\[
\mathbb{P}[\widehat\mu\le\mu-5\eta]\le\exp(-n\eta^2/2).
\]

Combining the two one-sided bounds with a union bound, we obtain
\[
    \mathbb{P}[\lvert\widehat\mu-\mu\rvert\ge5\eta]
    \le2\exp(-n\eta^2/2)\le\delta.
\]
This proves the lemma.
\end{proof}

\section{Continuous Time Versus Number of Updates}
\label{app:ct-dt-reduction}

The upper bounds in \Cref{thm:main-theorem-structure-learning,thm:main-theorem-parameter-learning}
are stated in terms of the continuous-time horizon $T$,
whereas the lower bound in \Cref{thm:infolb}
is stated in terms of the number of single-site updates.
This appendix records the standard reduction showing that these two formulations
are equivalent up to the factor $n$.

Recall from \Cref{def:cont-glauber-dynamics} that the jump times satisfy
$S^{(0)}=0$ and $S^{(\ell+1)}-S^{(\ell)}\sim \mathrm{Exp}(n)$ i.i.d.
Let
\[
N_T\coloneq \max\{\ell\ge 0:S^{(\ell)}\le T\}
\]
denote the number of updates by time $T$.
Then $N_T\sim \mathrm{Poisson}(nT)$.
Moreover, conditional on $N_T=m$, the continuous-time trajectory up to time $T$
is exactly the first $m$ steps of the embedded discrete-time chain,
together with the jump times $S^{(1)},\dots,S^{(m)}$.
Since these jump times are independent of $\Theta$ and of the embedded chain,
the only substantive difference between the two models is that in continuous time
one observes a \emph{random} number $N_T$ of discrete-time updates.

\begin{proposition}
\label{prop:ct-dt-equivalence}
The continuous-time and discrete-time formulations are interchangeable up to
constant factors.
\begin{enumerate}[label=(\roman*)]
    \item If there is an estimator that, from the first $N$ updates of the embedded
    discrete-time chain, succeeds with probability at least $1-\delta$,
    then there is an estimator that, from a continuous-time trajectory of length
    $2N/n$, succeeds with probability at least $1-\delta-e^{-N/4}$.

    \item If there is an estimator that, from a continuous-time trajectory of length
    $T$, succeeds with probability at least $1-\delta$,
    then there is an estimator that, from the first $\lceil 2nT\rceil$ updates of the
    embedded discrete-time chain, succeeds with probability at least
    $1-\delta-e^{-nT/3}$.
\end{enumerate}
\end{proposition}

\begin{proof}
For (i), observe the continuous-time trajectory up to time $2N/n$.
If $N_{2N/n}\ge N$, run the discrete-time estimator on the first $N$ updates of the
embedded chain and ignore the rest.
Since $N_{2N/n}\sim \mathrm{Poisson}(2N)$, the bad event is
$\{N_{2N/n}<N\}$, which has probability at most $e^{-N/4}$ by Chernoff.

For (ii), suppose we are given the first $M=\lceil 2nT\rceil$ updates of the
embedded discrete-time chain.
Sample i.i.d. waiting times $W_1,\dots,W_M\sim \mathrm{Exp}(n)$,
set $S^{(\ell)}=\sum_{r=1}^{\ell}W_r$,
and reconstruct the corresponding continuous-time path up to time $T$.
If $S^{(M)}>T$, this reconstruction is exact up to time $T$, so we may run the
continuous-time estimator.
The bad event is $S^{(M)}\le T$, equivalently $N_T\ge M$ for a rate-$n$ Poisson
process, and since $M\ge 2nT$ this has probability at most $e^{-nT/3}$ by Chernoff.
\end{proof}

\begin{corollary}
\label{cor:ct-dt-equivalence}
Up to absolute constant factors, the deterministic conversion between the two
models is $N\asymp nT$.
In particular, the lower bound $N=\Omega(n\log n/\alpha^2)$ of
\Cref{thm:infolb} is equivalent to a continuous-time lower bound
$T=\Omega(\log n/\alpha^2)$.
\end{corollary}

\section{A Technical Gap in the Analysis of Prior Work}
\label{sec:serious-gap}
We identify a gap in the analysis of \cite{tirukkonda2025structure}. We notified the authors,
{ who have since posted an updated manuscript on arXiv addressing the issue via an ``$iiji$''-based mechanism.}
Concretely, in the proof of Lemma~1 (step (c)), and again in the proof of Lemma~7 (Appendix~B.5),
the argument asserts that a certain conditional expectation of a ratio vanishes by invoking
mean-zero and (marginal) independence of Gaussian noise terms.
As we explain below, this cancellation is not justified when $(i,j)\in E$, because the denominator involves a future
increment of coordinate $j$ which depends on earlier noise injected at coordinate $i$.

From a technical viewpoint, this is one major reason that our mixing-free analysis relies on update patterns of the form ``$iiji$'' rather than on a direct ``$iji$'' ratio argument.
The additional update of $i$ is used to avoid exactly the type of dependence described above.
In the separate mixing-based result, we are able to work with ``$iji$'' patterns because the analysis there is different and does not rely on this cancellation.

We now explain the gap in detail.

In Lemma~1, they consider an update pattern $n_1<n_2<n_3$ in which node $i$ is updated at times $n_1$ and $n_3$ and node $j$ is updated at time $n_2$,
and they define a conditional expectation $\mathbb E_{\bar x,c}[\cdot]$ that fixes the values of $X_{N(i)\setminus\{j\}}$ at the beginning of the window
and conditions on the event $\lvert X_j^{(n_2)}-X_j^{(n_1)}\rvert\ge c$.
In the proof, the step labeled (c) asserts that
\[
\mathbb{E}_{\bar x,c}\!\left[
\frac{\varepsilon_i^{(n_3)}-\varepsilon_i^{(n_1)}}{X_j^{(n_2)}-X_j^{(n_1)}}
\right]=0,
\]
citing the mean-zero and (marginal) independence of the Gaussian noise variables $\{\varepsilon_i^{(t)}\}$.
A formally identical cancellation is used later in Appendix~B.5 (in the proof of Lemma~7) when rewriting their test statistic as
a signal term plus a ratio involving $\Delta\varepsilon_i/\Delta X_j$ and dropping the latter in conditional expectation.

\paragraph{Why the denominator depends on earlier noise when $(i,j)\in E$.}
Fix the ``$iji$'' update pattern $n_1<n_2<n_3$ from Lemma~1 of \cite{tirukkonda2025structure}, where $i$ is updated at times $n_1$ and $n_3$
and $j$ is updated at time $n_2$, and enforce their accompanying event that no neighbor of $i$ or $j$ (other than possibly each other) updates inside the window.
Assume $(i,j)\in E$, so $\Theta_{ji}\neq 0$.
Recall the Gaussian single-site update rule: when $u$ is updated at time $t$,
\begin{align*}
X_u^{(t)} &= -\!\sum_{k\in N(u)} \frac{\Theta_{uk}}{\Theta_{uu}}\,X_k^{(t-1)} + \varepsilon_u^{(t)},
\qquad\varepsilon_u^{(t)}\sim \mathcal{N}\!\left(0,\frac{1}{\Theta_{uu}}\right),
\end{align*}
Consider the randomness only through the lens of the single noise term $\varepsilon_i^{(n_1)}$ by fixing the pre-$n_1$ state, the update indices,
and all other Gaussian noises $\{\varepsilon_u^{(t)}:(u,t)\neq(i,n_1)\}$.
Under this fixing, the update at time $n_1$ gives an affine representation
\[
X_i^{(n_1)} = m_i + \varepsilon_i^{(n_1)}
\]
for a constant $m_i$ determined by what has been fixed.
Since $i$ is not updated between $n_1$ and $n_2$, the update of $j$ at time $n_2$ uses $X_i^{(n_1)}$, and therefore
\begin{align*}
X_j^{(n_2)} 
&=
-\frac{\Theta_{ji}}{\Theta_{jj}}\,X_i^{(n_1)}
-\sum_{k\in N(j)\setminus\{i\}}\frac{\Theta_{jk}}{\Theta_{jj}}\,X_k^{(n_2-1)}
+\varepsilon_j^{(n_2)} \\
&=
m_j' - \frac{\Theta_{ji}}{\Theta_{jj}}\,\varepsilon_i^{(n_1)} + \varepsilon_j^{(n_2)},
\end{align*}
for a constant $m_j'$ determined by the same fixing.
Consequently, the increment appearing in their denominator satisfies
\begin{align*}
X_j^{(n_2)}-X_j^{(n_1)} &= a + b\,\varepsilon_i^{(n_1)} + \varepsilon_j^{(n_2)},
\qquad b=-\frac{\Theta_{ji}}{\Theta_{jj}}\neq 0,
\end{align*}
for a constant $a$ determined by the fixed past.
In particular, when $(i,j)\in E$, the denominator is a function of $\varepsilon_i^{(n_1)}$, so it is not independent of the numerator noise term $\varepsilon_i^{(n_1)}$.

\paragraph{The conditioning $\lvert X_j^{(n_2)}-X_j^{(n_1)}\rvert\ge c$ does not repair the issue.}
In Lemma~1 (and again in Appendix~B.5), \cite{tirukkonda2025structure} conditions on the event
\[
\left\lvert X_j^{(n_2)}-X_j^{(n_1)}\right\rvert\ge c.
\]
Under the affine form above, this event is exactly
\[
\left\lvert a+b\,\varepsilon_i^{(n_1)}+\varepsilon_j^{(n_2)}\right\rvert\ge c,
\]
which depends on $\varepsilon_i^{(n_1)}$.
Thus, even after imposing the ``$\ge c$'' condition, the ratio term in step (c) involves a numerator noise component
that is statistically coupled to the denominator.

\paragraph{The ratio noise term need not have zero (conditional) expectation.}
As a result, the cancellation invoked in step (c) would require showing that an expression of the form
\[
\mathbb{E}\!\left[\frac{\varepsilon_i^{(n_1)}}{a+b\,\varepsilon_i^{(n_1)}+\varepsilon_j^{(n_2)}}\;\Bigg|\;\left\lvert a+b\,\varepsilon_i^{(n_1)}+\varepsilon_j^{(n_2)}\right\rvert\ge c\right]=0
\]
holds under the relevant conditioning.
There is no general reason for this to be true.

In fact, for $a=0$ and $b=1$, we have
\[\mathbb E\left[\frac{\varepsilon_i^{(n_1)}}{\varepsilon_i^{(n_1)}+\varepsilon_j^{(n_2)}}
\;\Bigg|\;\left\lvert\varepsilon_i^{(n_1)}+\varepsilon_j^{(n_2)}\right\rvert\ge c\right]=\frac12\]
by symmetry in $\varepsilon_i^{(n_1)}$ and $\varepsilon_j^{(n_2)}$,
as they are i.i.d.

Therefore, the vanishing of the ratio noise term cannot be justified solely from mean-zero and marginal independence of Gaussian noises.

\paragraph{Implications and connection to our approach.}
The discussion above shows that the specific cancellation used in step (c) of Lemma~1 of \cite{tirukkonda2025structure} is not justified as stated.
The denominator contains the future increment $X_j^{(n_2)}-X_j^{(n_1)}$, which, when $(i,j)\in E$, can depend on the earlier noise $\varepsilon_i^{(n_1)}$.
Consequently, the relevant ratio term need not have zero conditional expectation.
Since the same cancellation is used again in Appendix~B.5, in the proof of Lemma~7, the same issue propagates to the later separation argument built on that identity.

For our purposes, this explains why the mixing-free part of our analysis is based on $iiji$ patterns rather than on this direct $iji$ ratio argument.
The additional update of $i$ removes the dependence created by the first $i$-update before the later update of $j$ enters the statistic, so the conditional-independence step needed in our proofs becomes valid.
By contrast, when mixing is available, we also give a separate algorithm based on $iji$ patterns; that argument uses the mixing assumption and does not rely on the cancellation above.

\section{Non-Degeneracy Does Not Control Eigenvalues}
\label{app:nondegeneracy-conditioning}
In this appendix we demonstrate that the reciprocal of the minimum eigenvalue of the normalized precision matrix $\Theta'$ can be arbitrarily large even when all nonzero entries are bounded away from $0$. Further, we show that this phenomenon is independent of the sparsity constraint: even when the graph of $\Theta'$ is $d$-sparse for $d>2$, the edge-strength parameter $\alpha$ does not control $1/\lambda_{\min}(\Theta')$. Our construction is similar to Example~(5) in \cite{misra2020information}.

\begin{proposition}
    Fix $n \geq 2$, $\alpha <1$, and $d \geq 2$. For any large number $N\in \mathbb{R}^+$, there exists a matrix $\Theta'\in \mathbb{R}^{n\times n}$ with minimum edge strength $\alpha$ and sparsity $d$, while $\lambda_{\min}^{-1} > N$.
\end{proposition}

\begin{proof}
Consider the matrix
\[
B_{\varepsilon} \coloneq \begin{bmatrix}
    1 & 1-\varepsilon\\
    1 - \varepsilon & 1
\end{bmatrix}
\]
This matrix has minimum edge strength $1-\varepsilon$.

The matrix $B_{\varepsilon}$ has eigenvectors $v_1 = [1, 1]$ and $v_2 = [1, -1]$, satisfying $B_{\varepsilon}v_1 = (2 - \varepsilon)v_1$ and $B_{\varepsilon}v_2 = \varepsilon v_2$. Hence $\lambda_{\min}^{-1} = 1/\varepsilon$. This means that for all $\alpha < 1$ and any large enough number $M > N$ with $M > \frac{1}{1 - \alpha}$, the matrix $B_{1/M}$ has entries lower bounded by $1 - \frac{1}{M} > \alpha$, while the reciprocal of the minimum eigenvalue is $M$.

We can further extend the example of $B_\varepsilon$ to higher dimension to demonstrate that the edge-strength condition, together with sparsity constraints, can still allow large $\lambda_{\min}^{-1}$. For any $n > 2$, the block matrix
\[
B_{n, \varepsilon} \coloneq \begin{bmatrix}
    B_{\varepsilon} & 0\\
    0 & I_{n-2}
\end{bmatrix},
\]
where $I_{n-2}$ denotes the identity matrix in $n-2$ dimensions, has minimum entry $1 - \varepsilon$ and sparsity $d =2$. The eigenvalues of $B_{n, \varepsilon}$ are determined by its block components, so its minimum eigenvalue is $\varepsilon$ and its maximum eigenvalue is $2 - \varepsilon$; thus it has $\lambda_{\min}^{-1} = \frac{1}{\varepsilon}$. By the same argument as above, for any $\alpha < 1$ and $d > 2$, we may choose $M > \frac{1}{1-\alpha}$ and $M > N$, and then $B_{n,1/M}$ has $\lambda_{\min}^{-1}$ at least $M$ while having $\alpha$ edge strength and $d$-sparsity.
\end{proof}

\end{document}